\documentclass[letterpaper]{article} 
\usepackage{aaai23}  
\usepackage{times}  
\usepackage{helvet}  
\usepackage{courier}  
\usepackage[hyphens]{url}  
\usepackage{graphicx} 
\usepackage{natbib}  
\usepackage{caption} 
\usepackage{multirow}
\usepackage{algorithm}
\usepackage{algorithmic}
\usepackage{amsmath}
\usepackage{amsthm}
\usepackage{amsfonts}
\usepackage{bbold}
\usepackage{amssymb}
\usepackage{mathrsfs}
\usepackage{cleveref}
\usepackage{caption}
\usepackage{subcaption}
\newtheorem{theorem}{Theorem}
\newtheorem{example}{Example}
\newtheorem{lemma}[theorem]{Lemma}
\newtheorem{proposition}[theorem]{Proposition}
\newtheorem{assumption}[theorem]{Assumption}

\usepackage{newfloat}
\usepackage{listings}
\DeclareCaptionStyle{ruled}{labelfont=normalfont,labelsep=colon,strut=off} 
\floatstyle{ruled}
\newfloat{listing}{tb}{lst}{}
\floatname{listing}{Listing}
\title{Global Convergence of Two-Timescale Actor-Critic for Solving Linear Quadratic Regulator}
\author{
    Xuyang Chen\textsuperscript{\rm 1},
    Jingliang Duan\textsuperscript{\rm 2},
    Yingbin Liang\textsuperscript{\rm 3},
    Lin Zhao \textsuperscript{\rm 1}\footnote{Corresponding author}
}
\affiliations{
    \textsuperscript{\rm 1} National University of Singapore\\
    \textsuperscript{\rm 2} University of Science and Technology Beijing\\
    \textsuperscript{\rm 3} The Ohio State University\\
    chenxuyang@u.nus.edu, duanjl@ustb.edu.cn, liang.889@osu.edu, elezhli@nus.edu.sg
}

\usepackage{bibentry}

\begin{document}

\maketitle

\begin{abstract}
The actor-critic (AC) reinforcement learning algorithms have been the powerhouse behind many challenging applications. Nevertheless, its convergence is fragile in general. To study its instability, existing works mostly consider the uncommon double-loop variant or basic models with finite state and action space. We investigate the more practical single-sample two-timescale AC for solving the canonical linear quadratic regulator (LQR) problem, where the actor and the critic update only once with a single sample in each iteration on an unbounded continuous state and action space. Existing analysis cannot conclude the convergence for such a challenging case. We develop a new analysis framework that allows establishing the global convergence to an $\epsilon$-optimal solution with at most an $\mathcal{O}(\epsilon^{-2.5})$ sample complexity. To our knowledge, this is the first finite-time convergence analysis for the single sample two-timescale AC for solving LQR with global optimality. The sample complexity improves those of other variants by orders, which sheds light on the practical wisdom of single sample algorithms. We also further validate our theoretical findings via comprehensive simulation comparisons.
\end{abstract}

\section{Introduction}

\begin{table*}[t]
    \centering
    \begin{tabular}{ c| c|c |c}
 \hline
 Reference& Structure & Sample Complexity & Optimality\\
 \hline
 \multirow{2}{*}{\citet{xu2020non}} & \multirow{2}{*}{Multi-sample} & $\mathcal{O}(\epsilon^{-2.5})$& Local  \\
 \cline{3-4}& & $\mathcal{O}(\epsilon^{-4})$& Global \\
 \hline
 \citet{wu2020finite} & Single-sample & $\mathcal{O}(\epsilon^{-2.5})$ & Local   \\
 \hline
This paper & Single-sample & $\mathcal{O}(\epsilon^{-2.5})$ & Global  \\
 \hline
 \end{tabular}
    \caption{Comparison with other two-timescale actor-critic algorithms}
    \label{table1}
\end{table*}

The actor-critic (AC) methods \cite{konda1999actor} are among the most commonly used reinforcement learning (RL) algorithms, which have achieved tremendous empirical  successes \cite{mnih2016asynchronous,silver2017mastering}. In AC methods, the actor refers to the policy and the critic characterizes the action-value function (Q-function) given the actor. In each iteration, the critic tries to approximate the Q-function by applying policy evaluation algorithms \cite{dann2014policy,sutton2018reinforcement}, while the actor typically follows policy gradient \cite{sutton1999policy,agarwal2021theory} updates according to the Q-function provided by the critic. Compared with other RL algorithms, AC methods combine the advantages of both policy-based methods such as REINFORCE \cite{williams1992simple} and value-based methods such as temporal difference (TD) learning \cite{sutton1988learning,bhandari2018finite} and Q-learning \cite{watkins1992q}. Therefore, AC methods can be naturally applied to the continuous control setting \cite{silver2014deterministic} and meanwhile enjoy the low variance of bootstrapping.

Despite its empirical success, theoretical guarantees of AC still lag behind. Most existing works focus exclusively on the double-loop setting, where the critic updates many steps in the inner loop, followed by an actor update in the outer loop \cite{yang2019provably,agarwal2021theory,wang2019neural,abbasi2019politex,bhandari2021linear,xu2020improving}.  This setting yields accurate estimation of the Q-function and consequently the policy gradient. Therefore, double-loop setting decouples the convergence analysis of the actor and the critic, which further allows analyzing AC as a gradient method with error \cite{sutton1999policy,kakade2002approximately,shalev2014understanding,ruder2016overview}.

A more favorable implementation of AC in practice is the single-loop two-timescale setting, where the actor and critic are updated simultaneously in each iteration with different-timescale stepsizes. Typically, the actor stepsize is smaller than that of the critic. To establish the finite-time convergence of two-timescale AC methods, most existing results either focus on the multi-sample methods \cite{xu2020non,qiu2021finite} or the finite discrete action space \cite{wu2020finite}.  The former allows the critic to collect multiple samples to accurately  estimate the Q-function given the actor, which are rarely implemented in practice. It essentially decouples actor and critic in a similar way to the double-loop setting. We study the more practical single-sample AC algorithm similar to the one considered in~\citet{wu2020finite}, where the critic updates only once using a single sample per policy evaluation step. However, the latter only attains a stationary point under the finite-action space setting (see \Cref{table1} for comparisons). We address the important yet more challenging question: {\em can the single-sample two-timescale AC find a global optimal policy on the general unbounded continuous state-action space?} 

To this end, we analyze the classic single-sample AC for solving the Linear Quadratic Regulator (LQR), a fundamental control task which is commonly employed as a testbed to explore the behavior and limits of RL algorithms under continuous state-action spaces \cite{fazel2018global,yang2019provably,tu2018least,krauth2019finite,duan2022optimization}. 
In the LQR case, the Q-function is a linear function of the quadratic form of state and action. In general, it is difficult to establish the convergence of AC with linear function approximation \cite{bhandari2018finite}. In the double-loop and multi-sample settings \cite{yang2019provably,krauth2019finite}, the Q-function of LQR can be estimated arbitrarily accurately, and the LQR cost is guaranteed to converge to the global optimum monotonically per-iteration. However, the single-sample algorithm generally does not have these nice properties. It is more challenging to control the error propagation between the actor and the critic updates over iterations and prove its global convergence.

We distinguish our work from other model-free RL algorithms for solving LQR in \Cref{table2}. The zeroth-order methods and the policy iteration method are included for completeness. In particular, we note that~\citet{zhou2022single} analyzed the finite-time convergence under a single-timescale stepsize and multi-sample setting. The analysis requires the strong assumption on the uniform boundedness of the critic parameters. In comparison, our analysis does not require this assumption and considers the more challenging single-sample setting.

Within the literature of two-timescale AC for solving general MDP problems (see \Cref{table1}), we note that the analysis of~\citet{wu2020finite} depends critically on the assumptions of finite action space, bounded reward function, and bounded feature functions. However, these fundamental assumptions do not hold in the LQR case, making its analysis more challenging.

\subsection{Main Contribution}
Our main contributions are summarized as follows:

$\bullet$ Our work contributes to the theoretical understanding of AC methods. We for the first time show that the classic single-sample two-timescale AC can provably find the $\epsilon$-accurate global optimum with a sample complexity of $\mathcal{O}(\epsilon^{-2.5})$, under the continuous state-action space with linear function approximation. This is the same complexity order as those in \citet{wu2020finite,xu2020non}, but the latter only attain the local optimum.

We also adds to the work of RL on continuous control tasks. It is novel that even without completely solving the policy evaluation sub-problem  and only updating the actor with a biased policy gradient, the two-timescale AC algorithm can still find the global optimal policy for LQR, under common assumptions.
Our work may serve as the first step towards understanding the limits of single-sample two-timescale AC on continuous control tasks.

Compared with the state-of-the-art work of double-loop AC for solving LQR~\cite{yang2019provably}, we show the practical wisdom single-sample AC enjoys a lower sample complexity than  $\mathcal{O}(\epsilon^{-5})$ of the latter. We also show the algorithm is much more sample-efficient empirically compared to a few classic works.

$\bullet$ Technical-wise, despite the non-convexity of the LQR problem, we still find the global optimal policy under the single-sample update by exploiting the gradient domination property \cite{polyak1963gradient,nesterov2006cubic,fazel2018global}. Existing popular analysis~\cite{fazel2018global,yang2019provably} relies on the contraction of the cost learning errors. This nevertheless does not hold in the single-sample case. We alternatively establish the global convergence by showing the natural gradient of the objective function converges to zero and then using the gradient domination. Our work provides a more general proof framework for finding the optimal policy of LQR using various RL algorithms.

\begin{table*}[t]
\centering
\begin{tabular}{c|c|cc}
\hline
Reference & Algorithm & \multicolumn{2}{c}{Structure}                      \\ \hline
\citet{fazel2018global} & Zeroth-order & \multicolumn{2}{c}{\multirow{3}{*}{Double-loop}}     \\ \cline{1-2}
\citet{malik2019derivative} & Zeroth-order & \multicolumn{2}{c}{}                       \\ \cline{1-2}
\citet{yang2019provably} & Actor-Critic & \multicolumn{2}{c}{}                       \\ \hline
\citet{krauth2019finite} & Policy Iteration & \multicolumn{1}{c|}{\multirow{3}{*}{Single-loop}} & Multi-sample \\ \cline{1-2} \cline{4-4} 
\citet{zhou2022single} & Actor-Critic & \multicolumn{1}{c|}{}                   & Multi-sample (Single-timescale) \\ \cline{1-2} \cline{4-4} 
This paper & Actor-Critic & \multicolumn{1}{c|}{}                   & Single-sample (Two-timescale) \\ \hline
\end{tabular}
\caption{Comparison with other model-free RL algorithms for solving LQR.}
    \label{table2}
\end{table*}

\subsection{Related Work}
Due to the extensive studies on AC methods, we hereby review only those works that are mostly relevant to our study.

\textbf{Actor-Critic methods.} The AC algorithm was proposed in \citet{witten1977adaptive,sutton1984temporal,konda1999actor}.~\citet{kakade2001natural} extended it to the natural AC algorithm. The asymptotic convergence of AC algorithms has been well established in \citet{kakade2001natural,bhatnagar2009natural,castro2010convergent,zhang2020provably}. Many recent works focused on the finite-time convergence of AC methods. Under the double-loop setting, \citet{yang2019provably} established the global convergence of AC methods for solving LQR. \citet{wang2019neural} studied the global convergence of AC methods with both the actor and the critic being parameterized by neural networks. \citet{kumar2019sample} studied the finite-time local convergence of a few AC variants with linear function approximation, where the number of inner loop iterations 
grows linearly with the outer loop counting number. 

Under the two-timescale AC setting, \citet{khodadadian2022finite,hu2021actor} studied its finite-time convergence in tabular (finite state-action) case. For two-timescale AC with linear function approximation (see \Cref{table1} for a summary), \citet{wu2020finite} established the finite-time local convergence to a stationary point at a sample complexity of $\mathcal{O}(\epsilon^{-2.5})$. \citet{xu2020non} studied both local convergence and global convergence for two-timescale (natural) AC, with $\mathcal{O}(\epsilon^{-2.5})$ and $\mathcal{O}(\epsilon^{-4})$ sample complexity, respectively, under the discounted accumulated reward. The algorithm collects multiple samples to update the critic. 

Under the single-timescale setting, \citet{fu2020single} considered the regularized least-square temporal difference (LSTD) update for critic and established the finite-time convergence for both linear function approximation and nonlinear function approximation using neural networks.


\textbf{RL algorithms for LQR.} RL algorithms in the context of LQR have seen increased interest in the recent years. These works can be mainly divided into two categories: model-based methods \cite{dean2020sample,mania2019certainty,cohen2019learning,dean2018regret} and model-free methods. Our main interest lies in the model-free methods. Notably, \citet{fazel2018global} established the first global convergence result for LQR under the policy gradient method using derivative-free (one-point gradient estimator based zeroth-order) optimization at a sample complexity of $\mathcal{O}(\epsilon^{-4})$. \citet{malik2019derivative} employed two-point gradient estimator based zeroth-order optimization methods for solving LQR and improved the sample complexity to $\mathcal{O}(\epsilon^{-1})$. \citet{tu2019gap} characterized the sample complexity gap between model-based and model-free methods from an asymptotic viewpoint where their model-free algorithm is based on REINFORCE. 

Apart from policy gradient methods, \citet{tu2018least} studied the LSTD learning for LQR and derived the sample complexity to estimate the value function for a fixed policy. Subsequently, \citet{krauth2019finite} established the convergence and sample complexity of the LSTD policy iteration method under the LQR setting. On the subject of adopting AC to solve LQR, \citet{yang2019provably} provided the first finite-time analysis with convergence guarantee and sample complexity under the double-loop setting. For the more practical yet challenging single-sample two-timescale AC, there is no such theoretical guarantee so far, which is the focus of this paper.

\textbf{Notation.} For two sequences $\{ x_n \}$ and $\{ y_n \}$, we write $x_n=\mathcal{O}(y_n)$ if there exists an constant $C$ such that $x_n\leq Cy_n$. We use $\Vert \omega \Vert$ to denote the $\ell_2$-norm of a vector $\omega$ and use $\Vert A \Vert$ to denote the spectral norm of a matrix $A$. We also use $\Vert A\Vert_F$ to denote the Frobenius norm of a matrix $A$. We use $\sigma_{\min}(A)$ and $\sigma_{\max}(A)$ to denote the minimum and maximum singular values of a matrix $A$ respectively. We also use $\text{Tr}(A)$ to denote the trace of a square matrix $A$. For any symmetric matrix $M\in \mathbb{R}^{n\times n}$, let $\text{svec}(M)\in\mathbb{R}^{n(n+1)/2}$ denote the vectorization of the upper triangular part of $M$ such that $\Vert M \Vert^2_F = \langle \text{svec}(M),\text{svec}(M) \rangle $. Besides, let $\text{smat}(\cdot)$ denote the inverse of $\text{svec}(\cdot)$ so that $\text{smat}(\text{svec}(M))=M$. We denote by $A\otimes_s B$ the symmetric Kronecker product of two matrices $A$ and $B$.

\section{Preliminaries}
In this section, we introduce the AC algorithm and provide the theoretical background of LQR.
\subsection{Actor-Critic Algorithms}
Reinforcement learning problems can be formulated as a discrete-time Markov Decision Process (MDP), which is defined by $(\mathcal{X},\mathcal{U},\mathcal{P},c)$. Here $\mathcal{X}$ and $\mathcal{U}$ denote the state and the action space, respectively. At each time step $t$, the agent selects an action $u_t \in \mathcal{U}$ according to its current state $x_t\in \mathcal{X}$. In return, the agent will transit into the next state $x_{t+1}$ and receive a running cost $c(x_t,u_t)$. This transition kernel is defined by $\mathcal{P}$, which maps a state-action pair $(x_t, u_t)$ to the probability distribution over $x_{t+1}$. The agent's behavior is defined by a policy $\pi_{\theta}(u|x)$ parameterized by $\theta$, which maps a given state to a probability distribution over actions. In the following, we will use $\rho_{\theta}$ to denote the stationary state distribution induced by the policy $\pi_{\theta}$. 

The goal of the average RL is to learn a policy that minimizes the infinite-horizon time-average cost \cite{sutton1999policy,yang2019provably}, which is given by
\begin{align}\label{eq2.1.1}
    J(\theta):= \lim\limits_{T\to\infty}\mathbb{E}_{\theta}\frac{\sum_{t=0}^Tc(x_t,u_t)}{T}=\mathop{\mathbb{E}}_{x\sim \rho_{\theta},u\sim\pi_{\theta}}[c(x,u)],
\end{align}
where $\mathbb{E}_{\theta}$ denotes the expected value of a random variable whose state-action pair $(x_t,u_t)$ is obtained from policy $\pi_{\theta}$. Under this setting, the state-action value of policy $\pi_{\theta}$ can be calculated as 
\begin{align}\label{eq2.1.2}
    Q_{\theta}(x,u)=\sum\limits_{t=0}^{\infty}\mathbb{E}_{\theta}[c(x_t,u_t)-J(\theta)|x_0=x,u_0=u].
\end{align}

The typical AC consists of two alternate processes: (1) critic update, which estimates the Q-function $Q_{\theta}(x,u)$ of current policy $\pi_{\theta}$ using temporal difference (TD) learning \cite{sutton2018reinforcement}, and (2) actor update, which improves the policy to reduce the cost function $J(\theta)$ via gradient descent. By the policy gradient theorem \cite{sutton1999policy}, the gradient of $J(\theta)$ with respect to parameter $\theta$ is characterized by
\begin{align*}
    \nabla_{\theta}J(\theta)=\mathbb{E}_{x\sim\rho_{\theta},u\sim\pi_{\theta}}[\nabla_{\theta}\log\pi_{\theta}(u|x)\cdot Q_{\theta}(x,u)].
\end{align*}

One can also choose to update the policy using the natural policy gradient, which is the basic idea behind natural AC algorithms \cite{kakade2001natural}. The natural policy gradient is given by 
\begin{equation}
\label{eq.natural_pg}
    \nabla_{\theta}^N J(\theta)=F(\theta)^{\dagger}\nabla_{\theta}J(\theta).
\end{equation}
where 
\begin{align*}
    F(\theta)=\mathbb{E}_{x\sim\rho_{\theta},u\sim\pi_{\theta}}[\nabla_{\theta}\log\pi_{\theta}(u|x)(\nabla_{\theta}\log\pi_{\theta}(u|x))^\top]
\end{align*}
is the Fisher information matrix and $F(\theta)^{\dagger}$ denotes its Moore Penrose inverse.
\subsection{The Linear Quadratic Regulator Problem}\label{sec2.2}
As a canonical optimal control problem, the linear quadratic regulator (LQR) has become a convenient testbed for the optimization landscape analysis of various RL methods. In this paper, we aim to analyze the convergence performance of the AC algorithm applied to LQR. In particular, we consider a stochastic version of LQR (called noisy LQR), where the system dynamics and the running cost are  specified by
\begin{subequations}
\begin{align} 
\label{eq.state_dynamic}
x_{t+1}&=Ax_t+Bu_t+\epsilon_t,
 \\
 c(x,u)&=x^\top Qx+u^\top Ru.
\label{eq.cost_function}
\end{align} 
\end{subequations}
Here $x_t\in\mathbb{R}^d$ and $u_t\in\mathbb{R}^k$, $A\in \mathbb{R}^{d\times d}$ and $B\in \mathbb{R}^{d\times k}$ are system matrices, $Q\in\mathbb{S}^{d\times d}$ and $R\in\mathbb{S}^{k\times k}$ are performance matrices, and $\epsilon_t\sim \mathcal{N}(0,D_0)$ are i.i.d Gaussian random variables with $D_0>0$.


The goal of the noisy LQR problem is to find an action sequence $\{ u_t\}$ that minimizes the following infinite-horizon time-average cost
\begin{align*}
    \mathop{\text{minimize}}\limits_{\{ u_t\}}\quad &J(\{ u_t\}):=\lim\limits_{T\to\infty}\mathbb{E}[\frac{1}{T}\sum\limits_{t=1}^Tx_t^\top Qx_t+u_t^\top Ru_t]\\
    \text{subject to}\quad & \eqref{eq.state_dynamic}.
\end{align*}
From the optimal control theory \cite{anderson2007optimal,bertsekas2011approximate,bertsekas2019reinforcement}, the optimal policy is given by a linear feedback of the state
\begin{align}\label{eq2.2.1}
    u_t=-K^\ast x_t,
\end{align}
where $K^\ast \in\mathbb{R}^{k\times d}$ can be calculated as 
\begin{align*}
    K^\ast=(B^\top P^\ast B)^{-1}B^\top P^\ast A
\end{align*}
with $P^\ast$ being the unique solution to the following Algebraic Riccati Equation (ARE) \cite{anderson2007optimal}
\begin{align*}
    P^\ast=Q+A^\top P^\ast A-A^\top P^\ast B(B^\top P^\ast B+R)^{-1}B^\top P^\ast A.
\end{align*}

\subsection{Actor-critic for LQR}

Although the optimal solution of LQR can be easily found by solving the corresponding ARE, its solution relies on the complete model knowledge. In this paper, we pursue finding the optimal policy in a {\em model-free} way by using the AC method, without knowing or estimating $A,B,Q,R$.

 Based on the structure of the optimal policy in \eqref{eq2.2.1}, we parameterize the policy as
\begin{align}\label{policy}
\{\pi_K(\cdot|x)=\mathcal{N}(-Kx,\sigma^2I_k),K\in \mathbb{R}^{k\times d} \},
\end{align}
where $K$ is the policy parameter to be solved and $\sigma>0$ is the standard deviation of the exploration noise. In other words, given a state $x_t$, the agent will take an action $u_t$ according to $u_t=-Kx_t+\sigma \zeta_t$, where $\zeta_t\sim \mathcal{N}(0,I_k)$.
As a consequence, the closed-loop form of system \eqref{eq.state_dynamic} under policy \eqref{policy} is given by
\begin{align}\label{eq:6}
x_{t+1}=(A-BK)x_t+\xi_t,
\end{align}
where
\begin{align*}
\xi_t=\epsilon_t+\sigma B\zeta_t\sim \mathcal{N}(0,D_{\sigma})
\end{align*}
with $D_{\sigma}=D_0+\sigma^2BB^\top$. 

The set $\mathbb{K}$ of all stabilizing policies is given by
\begin{equation}
\label{eq:stabilizing-K}
\mathbb{K}:=\left\{K\in \mathbb{R}^{k\times d}:\rho(A-BK)<1\right\},
\end{equation}
where $\rho(\cdot)$ denotes the spectral radius. Before adopting AC to find the optimal policy $\pi_{K^\ast}$ that minimizes the corresponding cost $J(K)$ defined in \eqref{eq2.1.1}, we first need to establish the analytical formula of the average cost $J(K)$, the Q-function $Q_K(x, u)$, and the policy gradient $\nabla_K J(K)$ for a given stabilizing policy.  

It is well known that if $K\in\mathbb{K}$, the Markov chain in \eqref{eq:6} has a stationary distribution $\mathcal{N}(0,D_K)$, where $D_K$ satisfies the following Lyapunov equation
\begin{align}\label{lyap1}
    D_K=D_{\sigma}+(A-BK)D_K(A-BK)^\top.
\end{align}
Similarly, we define $P_K$ as the unique positive definite solution to
\begin{align}\label{lyap3}
    P_K=Q+K^\top RK+(A-BK)^\top P_K (A-BK).
\end{align}

Based on $D_K$ and $P_K$, the following lemma characterizes the explicit expression of $J(K)$ and its gradient $\nabla_K J(K)$. 
\begin{lemma}\cite{yang2019provably}\label{pro1}
For any $K\in \mathbb{K}$, the cost function $J(K)$ and its gradient $\nabla_K J(K)$ take the following forms
\begin{subequations}
\begin{align} 
\label{eq.cost_formula}
J(K)&= {\rm Tr}(P_KD_{\sigma})+\sigma^2 {\rm Tr}(R),\\
\nabla_K J(K)&=2E_KD_K,
\label{eq.gradient_formula}
\end{align} 
\end{subequations}
where $E_K:=(R+B^\top P_KB)K-B^\top P_KA$.
\end{lemma}
It can be shown that the natural gradient of $J(K)$ can be calculated as \cite{fazel2018global,yang2019provably}
\begin{align}
\label{eq.natural_formula}
\nabla_{K}^N J(K)=\nabla_{K}J(K)D_K^{-1}=E_K.
\end{align}
Note that we omit the constant coefficient since it can be absorbed by the stepsize. 

The expression of $Q_K(x,u)$ will also play an important role in our analysis later on.
\begin{lemma}\cite{bradtke1994adaptive,yang2019provably} \label{pro2}
For any $K\in \mathbb{K}$, the Q-function $Q_K(x,u)$ takes the following form
\begin{equation}
\label{eq.Q-structure}
\begin{aligned}
    Q_K(x,u)= & (x^\top,u^\top)\Omega_K\begin{pmatrix}
    x\\u
    \end{pmatrix} \\
    &-\sigma^2\text{\rm Tr}(R+P_KBB^\top)- \text{\rm Tr}(P_KD_K),
\end{aligned}
\end{equation}
where 
\begin{align}\label{eq:2}
\Omega_K:=
\begin{bmatrix}
\Omega^{11}_{K} & \Omega^{12}_K \\ \Omega^{21}_K & \Omega^{22}_K
\end{bmatrix}:=
\begin{bmatrix}
Q+A^\top P_K A & A^\top P_KB \\ B^\top P_K A & R+B^\top P_KB
\end{bmatrix}.
\end{align}
\end{lemma}

\section{Single-sample Natural Actor-Critic}
\label{sec:singlesample}
The expressions of $\nabla J(K)$ and $Q_K(x,u)$ in Lemmas \ref{pro1} and \ref{pro2} depend on $A$, $B$, $Q$, and $R$. For model-free learning, we establish a single-sample two-timescale AC algorithm for LQR in the following.

In view of the structure of the Q-function given in \eqref{eq.Q-structure}, we define the following feature functions,
\begin{align*}
\phi(x,u)=\text{svec}\left[\begin{pmatrix}
x\\u
\end{pmatrix}\begin{pmatrix}
x \\ u
\end{pmatrix}^\top \right].
\end{align*}
Then we can parameterize $Q_K(x,u)$ by the following linear function 
\begin{align*}
\hat{Q}_K(x,u;w)=\phi(x,u)^\top w + b.
\end{align*}
To drive $\hat{Q}_K(x,u;w)$ towards its true value  $Q_K(x,u)$ in a model-free way, the TD learning technique is applied to tune its parameters $\omega$:
\begin{equation}
\label{eq.TD_update}
\begin{aligned}
\omega_{t+1}&=\omega_{t} + \beta_t [(c_{t}-J(K)+\phi(x_{t+1},u_{t+1})^\top \omega_{t}\\
&\quad +b-\phi(x_{t},u_{t})^\top\omega_{t}-b)]\phi(x_{t},u_{t})\\
&=\omega_{t} + \beta_{t} [(c_{t}-J(K))\phi(x_{t},u_{t})\\
&\quad -\phi(x_{t},u_{t})(\phi(x_{t},u_{t})
-\phi(x_{t+1},u_{t+1}))^\top)\omega_{t}],
\end{aligned}
\end{equation}
where $\beta_t$ is the step size of the critic. 

To further simplify the expression, we denote $(x',u')$ as the subsequent state-action pair of $(x,u)$ and abbreviate $\mathbb{E}_{x\sim\rho_K,u\sim\pi_K(\cdot|x)}$ as $\mathbb{E}_{(x,u)}$. By taking the expectation of $\omega_{t+1}$  in \eqref{eq.TD_update} with respect to the stationary distribution, for any given $\omega_t$, the expected subsequent critic can be written as 
\begin{align}
\label{eq.expected_update}
\mathbb{E}[\omega_{t+1}|\omega_{t}]=\omega_{t}+\beta_t(b_K-A_K\omega_{t}),
\end{align}
where 
\begin{align}
A_K&=\mathbb{E}_{(x,u)}[\phi(x,u)(\phi(x,u)-\phi(x',u'))^\top)], \label{ak}\\
b_K&=\mathbb{E}_{(x,u)}[(c(x,u)-J(K))\phi(x,u)].\nonumber
\end{align}

Given a policy $\pi_K$, it is not hard to show that if the update in \eqref{eq.expected_update} has converged to some limiting point $\omega^\ast_K$, i.e., $\lim_{t\rightarrow \infty}\omega_t=\omega^\ast_K$, $\omega^\ast_K$ must be the solution of
\begin{align}\label{linear}
    A_K\omega=b_K.
\end{align}
We characterize the uniqueness and the explicit expression of $\omega^\ast_K$ in Proposition \ref{pro3}.
\begin{proposition}\label{pro3}
Suppose $K \in \mathbb{K}$. Then the matrix $A_K$ defined in \eqref{ak} is invertible such that the linear equation \eqref{linear} has an unique solution $\omega_K^\ast$, which is in the form of
\begin{align}
\label{eq.limiting_point}
    \omega^\ast_K ={\rm svec}(\Omega_K).
\end{align}
\end{proposition}

Combining \eqref{eq.natural_formula}, \eqref{eq:2},  and \eqref{eq.limiting_point}, we can express the natural gradient of $J(K)$ using only 
$\omega^\ast_K$:
\begin{align*}
\nabla_{K}^N J(K)=\Omega^{22}_K K-\Omega^{21}_K = \text{smat}(\omega^\ast_K)^{22}K-\text{smat}(\omega^\ast_K)^{21}
.\end{align*}
This enables us to estimate the natural policy gradient using the critic parameters $\omega_{t+1}$, and then update the actor in a model-free manner
\begin{align}
\label{eq.policy_update}
K_{t+1}=K_{t}-\alpha_t\widehat{\nabla_{K_t}^N J(K_t)},
\end{align}
where $\alpha_t$ is the (actor) step size and $\widehat{\nabla_{K_t}^N J(K_t)}$ is the natural gradient estimation depending on $\omega_{t+1}$:
\begin{align}
\label{eq.gradient_estimation}
\widehat{\nabla_{K_t}^N J(K_t)}=\text{smat}(\omega_{t+1})^{22}K_{t}-\text{smat}(\omega_{t+1})^{21}.
\end{align}

With the critic update rule \eqref{eq.TD_update} and the actor update rule \eqref{eq.policy_update} in place, we are ready to present the single-sample two-timescale natural AC algorithm for LQR.
\begin{algorithm}[H]
\caption{Single-Sample Two-timescale Natural Actor-Critic for Linear Quadratic Regulator}\label{alg1}            
\begin{algorithmic}[1]
\STATE \textbf{Input} initialize actor parameter $K_0 \in \mathbb{K}$, critic parameter $\omega_0$, average cost $\eta_0$, step sizes $\alpha_t$, $\beta_t$, and $\gamma_t$.
\FOR{$t=0,1,2,\cdots,T-1$}
    \STATE Sample $x_t$ from the stationary distribution $\rho_{K_t}$. 
    \STATE Take action $u_t\sim \pi_{K_t}(\cdot| x_t)$ and receive $c_t=c(x_t,u_t)$ and the subsequent state $x'_t$.
    \STATE Obtain $u'_t\sim \pi_{K_t}(\cdot| x'_t)$.
    \STATE TD error calculation
    $$\delta_t = c_{t} - \eta_{t}+\phi(x'_{t},u'_{t})^\top \omega_{t}-\phi(x_{t},u_{t})^\top\omega_{t}$$
    \STATE Average cost estimate
    $$\eta_{t+1}=\Pi_{U}(\eta_{t}+\gamma_t(c_{t}-\eta_{t}))$$
    \STATE Critic estimate
    $$\omega_{t+1}=\Pi_{\bar{\omega}}(\omega_{t} + \beta_t \delta_t \phi(x_{t},u_{t}))$$
    \STATE Actor update $$K_{t+1}=K_{t}-\alpha_t(\text{smat}(\omega_{t+1})^{22}K_{t}-\text{smat}(\omega_{t+1})^{21})$$
\ENDFOR
\end{algorithmic} 
\end{algorithm}
We call this algorithm ``single-sample'' because we only use exactly one sample to update the critic and the actor at each step. Line 3 of Algorithm \ref{alg1} samples from the stationary distribution corresponding to policy $\pi_K$, which is common in analysis of the LQR problem \cite{yang2019provably}. Such a requirement is only made to simplify the theoretical analysis. As shown in \citet{tu2018least}, when $K \in \mathbb{K}$, the Markov chain in \eqref{eq:6} is geometrically $\beta$-mixing and thus mixes quickly. Therefore, in practice, one can run the Markov chain in \eqref{eq:6} for a sufficient time and sample from the last one. 

Since the update of the critic parameter in \eqref{eq.TD_update} requires the knowledge of the average cost $J(K)$,  Line 7 is to provide an estimate of the cost function $J(K)$. Besides, compared with \eqref{eq.TD_update}, we introduce a projection operator in Line 8 to keep the critic norm-bounded, which is necessary to stabilize the algorithm and attain convergence guarantee. Similar operation has been commonly adopted in other literature \cite{wu2020finite,yang2019provably,xu2020non}. 
\section{Main Theory}
In this section, we establish the global convergence and analyze the finite-sample performance of \Cref{alg1}. All the corresponding proofs are provided in \citet{chen2022global}. 

Before preceding, the following assumptions are required, which are standard in the theoretical analysis of AC methods \cite{wu2020finite,fu2020single,yang2019provably,zhou2022single}.
\begin{assumption}\label{a1}
There exists a constant $\bar{K}>0$ such that $\Vert K_t\Vert\leq \bar{K}$ for all $t$.
\end{assumption}
The above assumes the uniform boundedness of the actor parameter. One can also ensure this by adding a projection operator to the actor. In this paper, we follow the previous works \cite{konda1999actor,bhatnagar2009natural,karmakar2018two,barakat2022analysis} to explicitly assume that our iterations remain bounded. As shown in our proof, it is only made to guarantee the uniform boundedness of the feature functions, which is a standard assumption in the literature of AC methods with linear function approximation \cite{xu2020non,wu2020finite,fu2020single}.  
\begin{assumption}\label{a2}
There exists a constant $\rho\in (0,1)$ such that $\rho(A-BK_t)\leq \rho$ for all $t$.
\end{assumption}
Assumption \ref{a2} is made to ensure the stability of the closed loop systems induced in each iteration and thus ensure the existence of the stationary distribution corresponding to policy $\pi_{K_t}$. 
In the single-sample case, the estimation of the natural gradient of $J(K)$ is biased and the policy change is noisy. Therefore, it is difficult to obtain a theoretical guarantee for this condition. Nevertheless, we will present numerical examples to support this assumption. Moreover, the assumption for the existence of stationary distribution is common and has been widely used in \citet{zhou2022single,olshevsky2022small}.

Under these two assumptions, we can now prove the convergence of Algorithm \ref{alg1}. 

We first establish the finite-time convergence of the critic learning.

\begin{theorem}\label{t1}
Suppose that Assumptions \ref{a1} and \ref{a2} hold. Choose $\alpha_t=\frac{c_{\alpha}}{(1+t)^\delta}, \beta_t=\frac{1}{(1+t)^v}, \gamma_t=\frac{1}{(1+t)^v}$, where $0<v<\delta<1$, $c_{\alpha}$ is a small positive constant. We have
\begin{align*}
    \frac{1}{T}\sum\limits_{t=0}^{T-1}\mathbb{E}[\Vert \omega_t-&\omega_{K_t}^\ast\Vert^2]\\=\ &\mathcal{O}(\frac{1}{T^{1-v}})+\mathcal{O}(\frac{1}{T^v})+\mathcal{O}(\frac{1}{T^{2(\delta-v)}}).
\end{align*}
\end{theorem}

Note that $\Vert \omega_t-\omega_{K_t}^\ast\Vert^2$ measures the difference between the estimated and true parameters of the corresponding Q-function under $K_t$. Despite the noisy single-sample critic update, this result establishes the convergence and characterizes the sample complexity of the critic for Algorithm \ref{alg1}. The complexity order depends on the selection of step sizes $\delta$ and $v$, which will be determined optimally later according to the finite-time bound of the actor.

Based on this finite-time convergence result of the critic,  we further characterize the global convergence of Algorithm \ref{alg1} below.
\begin{theorem}\label{t2}
Suppose that Assumptions \ref{a1} and \ref{a2} hold. Choose $\alpha_t=\frac{c_{\alpha}}{(1+t)^\delta}, \beta_t=\frac{1}{(1+t)^v}, \gamma_t=\frac{1}{(1+t)^v}$, where $0<v<\delta<1$, $c_{\alpha}$ is a small positive constant. We have
\begin{align*}
    \mathop{\min}\limits_{0\leq t< T} \mathbb{E}[J(K_t)&-J(K^\ast)]\\
    &=\mathcal{O}(\frac{1}{T^{1-\delta}})+\mathcal{O}(\frac{1}{T^v})+\mathcal{O}(\frac{1}{T^{2(\delta-v)}}).
\end{align*}
\end{theorem}

The optimal convergence rate of the actor is attained at $\delta = \frac{3}{5}$ and $v=\frac{2}{5}$. In particular, to obtain an $\epsilon$-optimal policy, the optimal complexity of Algorithm \ref{alg1} is $\mathcal{O}(\epsilon^{-2.5})$. To our knowledge, this is the first convergence result for solving LQR using single-sample two-timescale AC method.

To see the merit of our proof framework, we sketch the main proof steps of Theorems \ref{t1} and \ref{t2} in the following. The supporting propositions and theorems mentioned below can be found in \citet{chen2022global}. Note that since the critic and the actor are coupled together, we define the following notations to clarify their dependency:
\begin{align*}
    A(T)=\ &\frac{1}{T}\sum\limits_{t=0}^{T-1} \mathbb{E}[(\eta_t-J(K_t))^2],\\
    B(T)=\ &\frac{1}{T}\sum\limits_{t=0}^{T-1} \mathbb{E}[\Vert \omega_t-\omega_{K_t}^\ast\Vert^2],\\
    C(T)=\ &\frac{1}{T}\sum\limits_{t=0}^{T-1} \mathbb{E}[\Vert \nabla_{K_t}^N J(K_t)\Vert^2],
\end{align*}
where $\nabla_{K_t}^N J(K_t)=E_{K_t}$ is defined in \eqref{eq.natural_formula}.

\noindent \textbf{Proof Sketch:} 
\begin{enumerate}
\item Prove the convergence of the average cost. Note that in Line 7 of Algorithm \ref{alg1}, the average cost estimator $\eta_t$ is only coupled with actor $K_t$ via the cost $c_{t}$. We bound $\eta_t$ utilizing the local Lipschitz continuity of $J(K)$ shown in \citet[Proposition 12]{chen2022global} and the boundedness of $K_t$. Then it can be proved that
\begin{align*}
    A(T)\leq \sqrt{A(T)}\mathcal{O}(T^{\frac{1}{2}-2(\delta-v)})+\mathcal{O}(\frac{1}{T^{1-v}})+\mathcal{O}(\frac{1}{T^v}),
\end{align*}
where $\mathcal{O}(T^{\frac{1}{2}-2(\delta-v)})$ comes from the ratio between actor step size and critic step size, which reveals how the merit of two-timescale method can contribute to the proof. The other two terms $\mathcal{O}(\frac{1}{T^{1-v}})$ and $\mathcal{O}(\frac{1}{T^v})$ are induced by the step size for the average cost, which is $\gamma_{t}=\frac{1}{(1+t)^v}$. Solving this inequality gives the convergence of $\eta_t$ which we presented in \citet[Theorem 13]{chen2022global}.

\item Prove the convergence of the critic. Note that the critic is coupled with both $\eta_t$ and actor $K_t$. We decouple the critic and the actor in a similiar way to step 1 utilizing the Lipschitz continuity of $\omega^\ast_t$ as shown in \citet[Proposition 15]{chen2022global}. Then, the following inequality can be obtained,
\begin{align*}
    B(T)\leq &\sqrt{A(T)B(T)}+\sqrt{\mathcal{O}(\frac{1}{T^{2(\delta-v)}})B(T)}\\
    &+\mathcal{O}(\frac{1}{T^{1-v}})+\mathcal{O}(\frac{1}{T^v}),
\end{align*}
where $\sqrt{A(T)B(T)}$ shows the coupling between the average cost estimator $\eta_t$ and the critic $\omega_t$. Terms 
$\mathcal{O}(\frac{1}{T^{2(\delta-v)}})$, $\mathcal{O}(\frac{1}{T^{1-v}})$ and $\mathcal{O}(\frac{1}{T^v})$ are induced by the stepsizes. Combining the bound for $A(T)$, we can conclude the convergence of critic detailed in Theorem \ref{t1}. 

\item Prove the convergence of the actor. We utilize the almost smoothness property of the cost function $J(K)$ to establish the relation between actor, critic, and the natural gradient. We first prove that 
\begin{align*}
    C(T)\leq \sqrt{B(T)C(T)}+\mathcal{O}(\frac{1}{T^{1-\delta}})+\mathcal{O}(\frac{1}{T^\delta}),
\end{align*}
where $\sqrt{B(T)C(T)}$ shows the coupling between critic and actor. The terms $\mathcal{O}(\frac{1}{T^{1-\delta}})$ and $\mathcal{O}(\frac{1}{T^\delta})$ are induced by the step size of actor. By the convergence of critic established in Theorem \ref{t1}, we can show that $C(T)$ converges to zero, which means the convergence of natural gradient. Finally, using the following gradient domination condition (see \citet[Lemma 17]{chen2022global})
\begin{align*}
    J(K)-J(K^\ast)
    \leq &\frac{1}{\sigma_{\text{min}}(R)}\Vert D_{K^\ast}\Vert \text{Tr}(E_K^\top E_K),
\end{align*}
we can further prove the global convergence of actor shown in Theorem \ref{t2}.
\end{enumerate}




\section{Experiments}
In this section, we provide two examples to validate our theoretical results. 
\begin{example}
\label{ex.1}
Consider a two-dimensional system with 
\begin{align*}
    \text{A}=\begin{bmatrix}
    0 & 1\\ 1 & 0
    \end{bmatrix}, \text{B}=\begin{bmatrix}
    0 & 1\\ 1 & 0
    \end{bmatrix}, \text{Q}=\begin{bmatrix}
    9 & 2\\ 2 & 1
    \end{bmatrix}, \text{R}=\begin{bmatrix}
    1 & 2\\ 2 & 8
    \end{bmatrix}.
\end{align*}
\end{example}

\begin{example}
\label{ex.2}
Consider a four-dimensional system with 
\begin{align*}
    &\text{A}=\begin{bmatrix}
    0.2 & 0.1 & 1 & 0\\ 0.2 & 0.1 & 0.1 & 0\\ 0 & 0.1 & 0.5 & 0\\ 0 & 0 & 0 & 0.5
    \end{bmatrix}, \text{B}=\begin{bmatrix}
    0.3 & 0 & 0\\ 0.2 & 0 & 0.3\\ 1 & 1 & 0.3 \\ 0.3 & 0.1 & 0.1
    \end{bmatrix},  \\
    &\text{Q}=\begin{bmatrix}
    1 & 0 & 0.2 & 0\\ 0 & 1 & 0.1 & 0\\ 0.2 & 0.1 & 1 & 0.1 \\ 0 & 0 & 0.1 & 1
    \end{bmatrix}, \text{R}=\begin{bmatrix}
    1 & 0.1 & 1\\ 0.1 & 1 & 0.5 \\ 1 & 0.5 & 2
    \end{bmatrix}.
\end{align*}
\end{example}

The learning results of Algorithm \ref{alg1} for these two examples are shown in Figure \ref{fig1}. Consistent with our theoretical analysis, both the critic and the actor gradually converge to the optimal solution. Interested readers can refer to the Appendix of \citet{chen2022global} for experimental details.

\begin{figure}[H]
     \centering
     \begin{subfigure}[b]{0.23\textwidth}
         \centering
         \includegraphics[width=\textwidth]{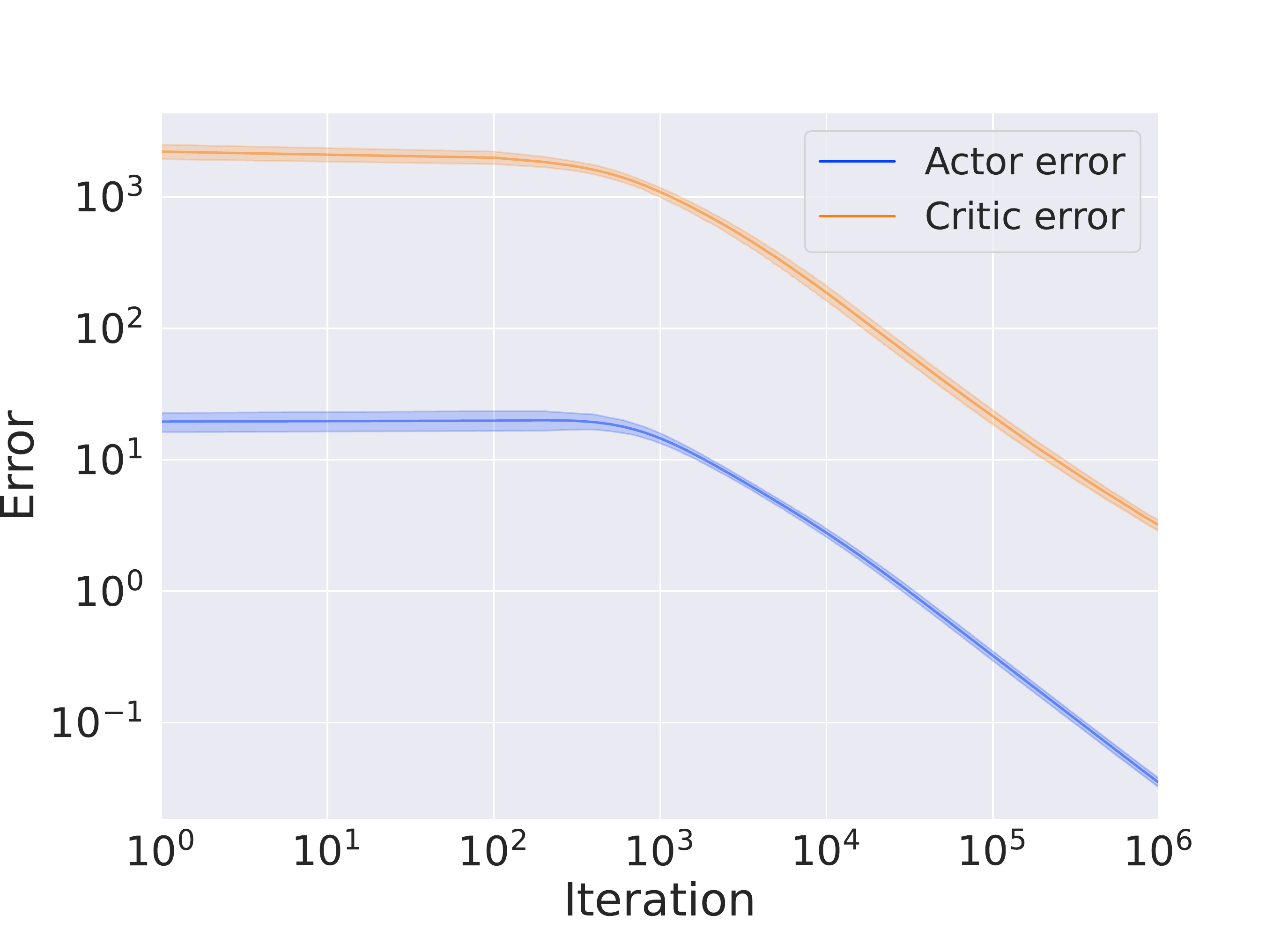}
         \caption{Example \ref{ex.1}}
     \end{subfigure}
     \hfill
     \begin{subfigure}[b]{0.23\textwidth}
         \centering
         \includegraphics[width=\textwidth]{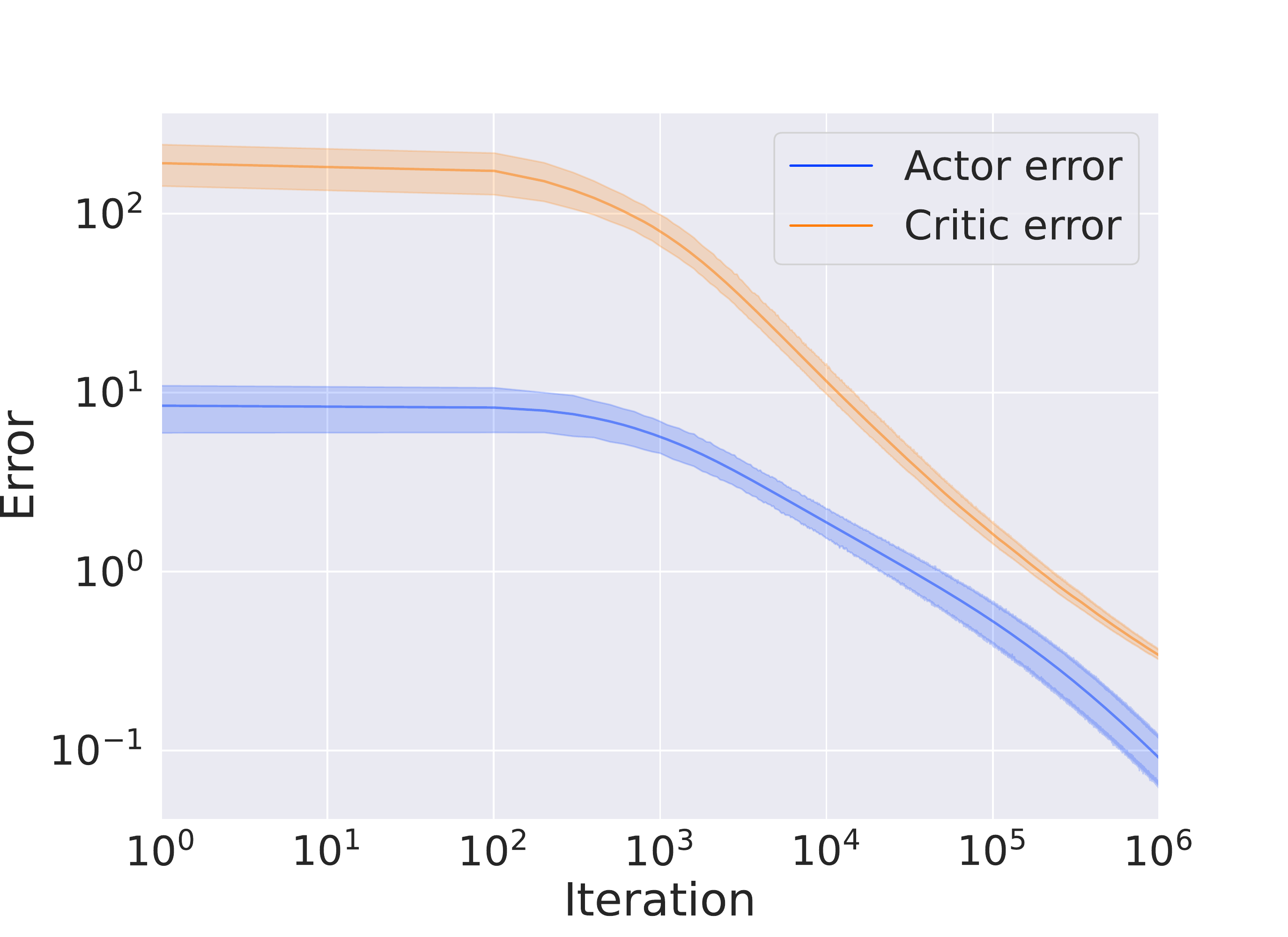}
         \caption{Example \ref{ex.2}}
     \end{subfigure}
        \caption{Learning curves of the critic and the actor. Critic error refers to $\frac{1}{T}\sum_{t=0}^{T-1}\Vert \omega_t-\omega_{K_t}^\ast\Vert^2$ while actor error refers to $\frac{1}{T}\sum_{t=0}^{T-1}[J(K_t)-J(K^\ast)]$. The solid lines correspond to the mean and the shaded regions correspond to 95\% confidence interval over 10 independent runs.}
        \label{fig1}
\end{figure}


We also compare our algorithm with the double-loop AC algorithm proposed in \citet{yang2019provably} and the zeroth-order method described in \citet{fazel2018global}. We plotted the relative error of the actor parameters for all three methods in Figure \ref{fig2}. These simulation results show the superior sample-efficiency of Algorithm \ref{alg1} empirically, confirming the practical wisdom of single sample two-timescale AC method. 

\begin{figure}[H]
     \centering
     \begin{subfigure}[b]{0.23\textwidth}
         \centering
         \includegraphics[width=\textwidth]{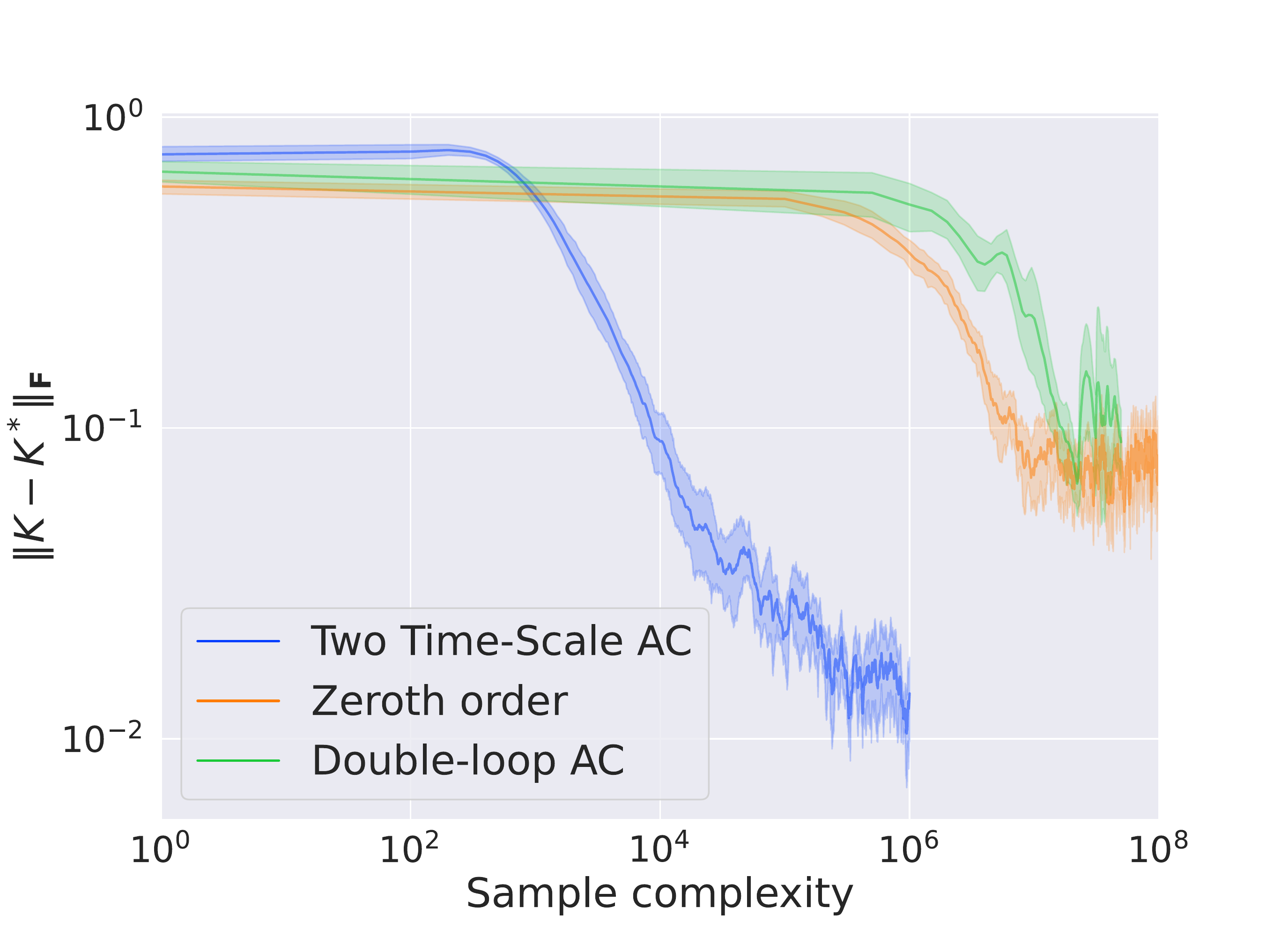}
         \caption{Example \ref{ex.1}}
         \label{fig:y equals x}
     \end{subfigure}
     \hfill
     \begin{subfigure}[b]{0.23\textwidth}
         \centering
         \includegraphics[width=\textwidth]{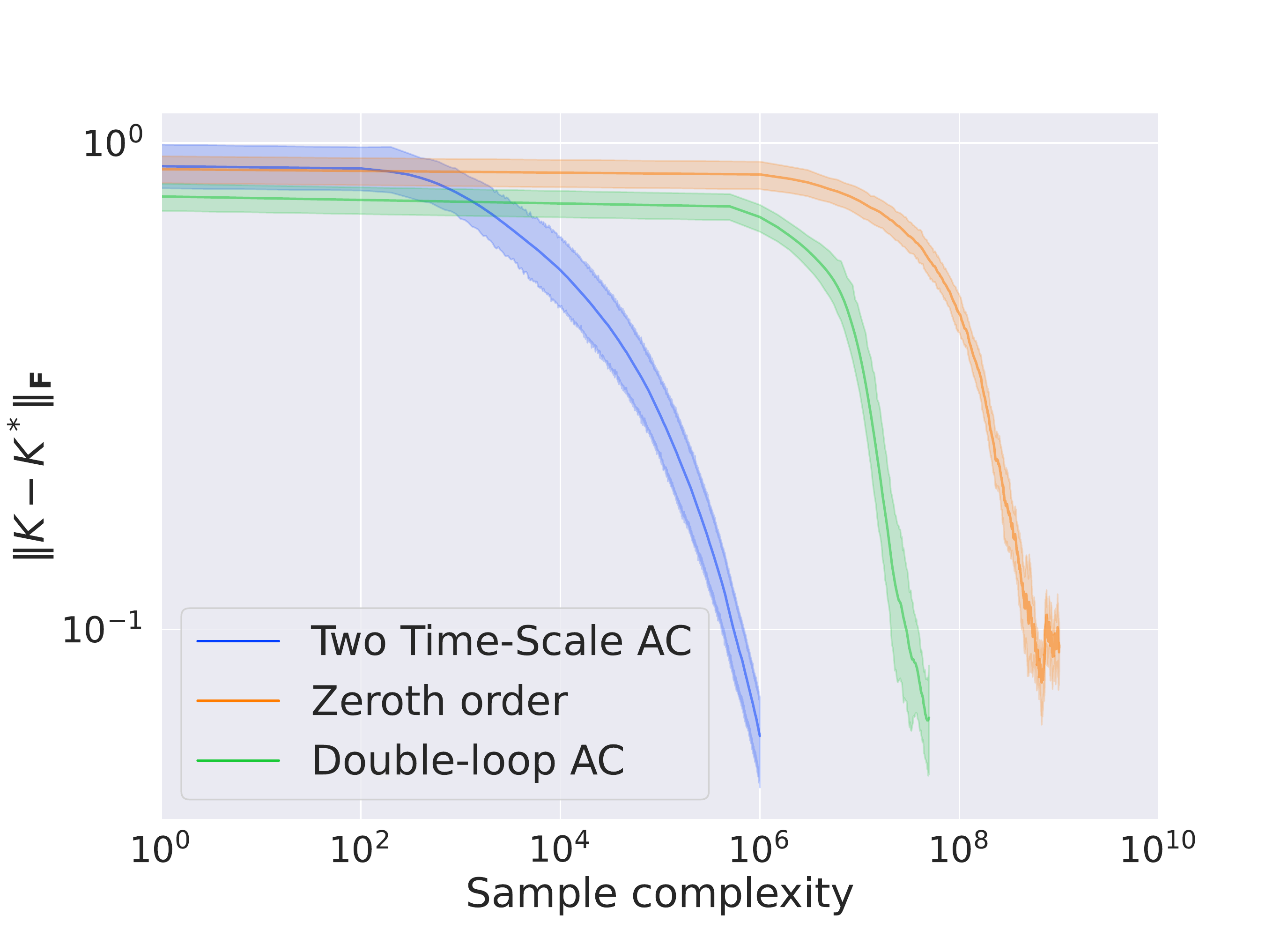}
         \caption{Example \ref{ex.2}}
         \label{fig:three sin x}
     \end{subfigure}
        \caption{Sample complexity comparison. The solid lines correspond to the mean and the shaded regions correspond to 95\% confidence interval over 10 independent runs.}
        \label{fig2}
\end{figure}

\section{Conclusion and Discussion}
In this paper, we establish the first finite-time global convergence analysis for the two-timescale AC method under LQR setting. We adopt a more practical single-sample two-timescale AC method and achieve an $\mathcal{O}(\epsilon^{-2.5})$ sample complexity. Our proof techniques of decoupling the actor and critic updates and controlling the accumulative estimate errors of the actor induced by the critic are novel and applicable to analyzing other AC methods where the actor and critic are updated simultaneously. 
Our future work includes further tightening the sample complexity bound under more relaxed settings and assumptions. 

\subsubsection*{Acknowledgments}
 The work of X. Chen and L. Zhao was supported in part by the Singapore Ministry of Education Academic Research Fund Tier 1 (R-263-000-E60-133). The work of J. Duan is sponsored by NSF China with 52202487. The work of Y. Liang was supported in part by the U.S. National Science Foundation under the grant CCF-1761506.

\clearpage
\newpage
\appendix
\section{Proof of Main Theorems}\label{appendix1}
\subsection{Convergence of Critic}
We establish the convergence of critic first. In our two-timescale algorithm, the actor updates simultaneously with critic, so we define the following notations for clarification.
\begin{align*}
\omega^\ast_t&:=\omega^\ast_{K_t},\\
z_t&:=\omega_t-\omega^\ast_t,\\
y_t&:=\eta_t-J(K_t),\\
O_t&:=(x_t,u_t,x'_{t},u'_{t}),\\
\Delta g(O,\eta,K)&:=[J(K)-\eta]\phi(x,u),\\
g(O,\omega,K)&:=[c(x,u)-J(K)+(\phi(x',u')\\
&\quad \ \ -\phi(x,u))^\top \omega]\phi(x,u),\\
\bar{g}(\omega,K)&:=\mathbb{E}_{x\sim \rho_K,u\sim \pi_K}[[c(x,u)-J(K)\\
&\quad \ \ +(\phi(x',u')-\phi(x,u))^\top \omega]\phi(x,u)].\\
\Lambda(O,\omega,K)&:=\langle \omega-\omega^\ast_K,g(O,\omega,K)-\bar{g}(\omega,K)\rangle.
\end{align*}
\subsubsection{Approximating the Average Cost}
In this section, we establish the finite-time convergence for the average cost estimator $\eta_t$.

We first give an uniform upper bound for the covariance matrix $D_{K_t}$.
\begin{proposition}\label{pa1.2.1}
(Upper bound for covariance matrix). Suppose that Assumption \ref{a2} holds. The covariance matrix of the stationary distribution $\mathcal{N}(0,D_{K_t})$ induced by the Markov chain in \eqref{eq:6} can be upper bounded by
\begin{align}\label{updk}
    \Vert D_{K_t}\Vert \leq \frac{c_1\Vert D_{\sigma}\Vert}{1-(\frac{1+\rho}{2})^2}\ \text{for all}\ t,
\end{align}
where $c_1$ is a constant.
\end{proposition}
Note that the sampled state-action pair $(x_t,u_t)$ can be unbounded. However, in the following proposition, we show that by taking expectation over the stationary state-action distribution, the expected cost and feature function are all bounded.
\begin{proposition}[Upper bound for reward and feature function]\label{new_bounded}
    For $t=0,1,\cdots, T-1$, we have
    \begin{align*}
        \mathbb{E}[c^2_t]\leq C,\\
        \mathbb{E}[\Vert \phi(x_t,u_t)\Vert^2]\leq C,
    \end{align*}
where $C$ is a constant.
\end{proposition}
\begin{proposition}[Upper bound for cost function]\label{cost_function}
    For $t=0,1,\cdots, T-1$, we have
    \begin{align*}
        J(K_t)\leq U,
    \end{align*}
where $U:=\Vert Q\Vert_{\text{F}}+d\Bar{K}^2+\Vert R\Vert_{\text{F}}+\sigma^2\text{Tr}(R)+\frac{c_1\sqrt{d}\Vert D_{\sigma}\Vert}{1-(\frac{1+\rho}{2})^2}$ is a constant.
\end{proposition}
Proposition \ref{cost_function} justifies the projection introduced in the average cost update since the true average cost $J(K_t)$ is upper bounded by $U$.
\begin{lemma}\label{l2}
(Perturbation of $P_K$). Suppose $K'$ is a small perturbation of $K$ in the sense that
\begin{align}
    \Vert K'-K\Vert \leq \frac{\sigma_{\text{min}}(D_0)}{4}\Vert D_K\Vert^{-1}\Vert B\Vert^{-1}(\Vert A-BK\Vert+1)^{-1}. \label{eq24}
\end{align}
Then we have 
\begin{align*}
    \Vert P_{K'}-P_K\Vert \leq &6\sigma_{\text{min}}^{-1}(D_0)\Vert D_K\Vert \Vert K\Vert \Vert R\Vert(\Vert K\Vert \Vert B\Vert \cdot \\
    &\Vert A-BK\Vert +\Vert K\Vert \Vert B\Vert +1)\Vert K-K'\Vert.
\end{align*}
\end{lemma}
With the perturbation of $P_K$, we are ready to prove the Lipschitz continuous of $J(K)$.
\begin{proposition}\label{p1}
(Local Lipschitz continuity of $J(K)$) Suppose Lemma \ref{l2} holds, for any $K_t,K_{t+1}$, we have
\begin{align*}
    |J(K_{t+1})-J(K_t)|\leq l_1 \Vert K_{t+1}-K_t\Vert,
\end{align*}
where
\begin{align}\label{newl1}
    l_1:=&6c_1d\bar{K}\sigma_{\text{min}}^{-1}(D_0)\frac{\Vert D_{\sigma} \Vert^2}{1-(\frac{1+\rho}{2})^2}\Vert R\Vert(\bar{K}\Vert B\Vert \cdot \nonumber\\
    &(\Vert A\Vert +\bar{K}\Vert B\Vert+1) +1).
\end{align}
\end{proposition}
Equipped with the above propositions and lemmas, we are able to show that $\eta_t$ can approximate $J(K_t)$ by the following theorem.
\begin{theorem}\label{t3}
Suppose that Assumptions \ref{a1} and \ref{a2} hold. We choose $\alpha_t=\frac{c_{\alpha}}{(1+t)^\delta}, \beta_t=\frac{1}{(1+t)^v}, \gamma_t=\frac{1}{(1+t)^v}$, where $0<v<\delta<1$, $c_{\alpha}$ is a small positive constant. We have
\begin{align}
\frac{1}{T}{\sum\limits_{t=0}^{T-1}\mathbb{E}y^2_t}= \mathcal{O}(\frac{1}{T^{1-v}})+\mathcal{O}(\frac{1}{T^v})+\mathcal{O}(\frac{1}{T^{\delta-v}}).
\end{align}
\end{theorem}
\begin{proof}
From Line 5 of Algorithm \ref{alg1}, we have
\begin{align*}
    \eta_{t+1}-J(K_{t+1})=\ &\Pi_{U}(\eta_t+\gamma_t(c_t-\eta_t))-J(K_{t+1})\\
    =\ &\Pi_{U}(\eta_t+\gamma_t(c_t-\eta_t))-\Pi_{U}(J(K_{t+1})).
\end{align*}
Then, it can be shown that
\begin{align*}
    |y_{t+1}|=\ &|\Pi_{U}(\eta_t+\gamma_t(c_t-\eta_t))-\Pi_{U}(J(K_{t+1}))|\\
    \leq\  & |\eta_t+\gamma_t(c_t-\eta_t)-J(K_{t+1})|\\
    =\ & |y_t+J(K_t)-J(K_{t+1})+\gamma_t(c_t-\eta_t)|.
\end{align*}
Thus we get
\begin{align*}
    y^2_{t+1}&\leq(y_t+J(K_t)-J(K_{t+1})+\gamma_t(c_t-\eta_t))^2\\
    &\leq y^2_t+2\gamma_ty_t(c_t-\eta_t)+2y_t(J(K_t)-J(K_{t+1}))\\
    &\quad +2(J(K_t)-J(K_{t+1}))^2+2\gamma^2_t(c_t-\eta_t)^2\\
    &=(1-2\gamma_t)y^2_t+2\gamma_t y_t(c_t-J(K_t))+2\gamma_t^2(c_t-\eta_t)^2\\
    &\quad +2y_t(J(K_t)-J(K_{t+1}))+2(J(K_t)-J(K_{t+1}))^2.
\end{align*}
Taking expectation up to $(x_{t},u_{t})$ for both sides, we have
\begin{align*}
    \mathbb{E}[y_{t+1}^2]\leq\ &(1-2\gamma_t)\mathbb{E}[y^2_t]+2\mathbb{E}[y_t(J(K_t)-J(K_{t+1})]\\
    &+2\mathbb{E}[(J(K_t)-J(K_{t+1}))^2]+2\gamma_t^2\mathbb{E}[(c_t-\eta_t)^2]\\
    &+2\gamma_t\mathbb{E}[y_t(c_t-J(K_t))].
\end{align*}
To compute $\mathbb{E}[y_t(c_t-J(K_t))]$, we use the notation $\vartheta_t$ to denote the vector $(x_t,u_t)$ and $\vartheta_{0:t}$ to denote the sequence $(x_0,u_0),(x_1,u_1),\cdots,(x_{t},u_{t})$. Hence, we have
\begin{align*}
    &\mathbb{E}[y_t(c_t-J(K_t))]\\
    =\ &\mathbb{E}_{\vartheta_{0:t}}[y_t(c_t-J(K_t))]\\
    =\ &\mathbb{E}_{\vartheta_{0:t-1}}\mathbb{E}_{\vartheta_{0:t}}[y_t(c_t-J(K_t))|\vartheta_{0:t-1}].
\end{align*}
Once we know $\vartheta_{0:t-1}$, $y_t$ is not a random variable any more. Thus we get
\begin{align*}
    &\mathbb{E}_{\vartheta_{0:t-1}}\mathbb{E}_{\vartheta_{0:t}}[y_t(c_t-J(K_t))|\vartheta_{0:t-1}]\\
    =\ &\mathbb{E}_{\vartheta_{0:t-1}}y_t\mathbb{E}_{\vartheta_{0:t}}[(c_t-J(K_t))|\vartheta_{0:t-1}]\\
    =\ &\mathbb{E}_{\vartheta_{0:t-1}}y_t\mathbb{E}_{\vartheta_t}[(c_t-J(K_t))|\vartheta_{0:t-1}]\\
    =\ &0.
\end{align*}
Combining the fact $2\gamma_t\mathbb{E}[y_t(c_t-J(K_t))]=0$, we get
\begin{align*}
    \mathbb{E}[y_{t+1}^2]\leq&(1-2\gamma_t)\mathbb{E}[y^2_t]+2\mathbb{E}[y_t(J(K_t)-J(K_{t+1}))]\\
    &+2\mathbb{E}[(J(K_t)-J(K_{t+1}))^2]+2\gamma_t^2\mathbb{E}[(c_t-\eta_t)^2].
\end{align*}

Rearranging and summing from $0$ to $T-1$, we have
\begin{align*}
    \sum\limits_{t=0}^{T-1}\mathbb{E}[y^2_t]&\leq  \underbrace{\sum\limits_{t=0}^{T-1}\frac{1}{2\gamma_t}\mathbb{E}[(y^2_t-y^2_{t+1})]}_{I_1}\\
    &+\underbrace{\sum\limits_{t=0}^{T-1}\frac{1}{\gamma_t}\mathbb{E}[y_t(J(K_t)-J(K_{t+1}))}_{I_2}\\
    & +\underbrace{\sum\limits_{t=0}^{T-1}\frac{1}{\gamma_t}\mathbb{E}[(J(K_t)-J(K_{t+1}))^2]}_{I_3}\\
    &+\underbrace{\sum\limits_{t=0}^{T-1}\gamma_t\mathbb{E}[(c_t-\eta_t)^2}_{I_4}.
\end{align*}
In the sequel, we need to control $I_1,I_2,I_3,I_4$ respectively. For $I_1$, following Abel summation by parts, we have
\begin{align*}
    I_1&=\sum\limits_{t=0}^{T-1}\frac{1}{2\gamma_t}\mathbb{E}[(y^2_t-y^2_{t+1})]\\
    &=\sum\limits_{t=1}^{T-1}(\frac{1}{2\gamma_t}-\frac{1}{2\gamma_{t-1}})\mathbb{E}[y^2_t]+\frac{1}{2\gamma_{0}}\mathbb{E}[y^2_{0}]-\frac{1}{2\gamma_{T-1}}\mathbb{E}[y^2_{T}]\\
    &\leq 4U^2\sum\limits_{t=1}^{T-1}(\frac{1}{2\gamma_t}-\frac{1}{2\gamma_{t-1}})+\frac{1}{2\gamma_{0}}4U^2\\
    &\leq \frac{2U^2}{\gamma_{T-1}}.
\end{align*}
For $I_2$, we need to verify Lemma \ref{l2} first and use the local Lipschitz continuous property of $J(K)$ provided by Proposition \ref{p1} to bound this term. 
Since we have
\begin{align*}
    \Vert K_{t+1}-K_{t}\Vert = \alpha_t\Vert (\text{smat}(\omega_{t+1})^{22}K_{t}-\text{smat}(\omega_{t+1})^{21})\Vert,
\end{align*}
to satisfy \eqref{eq24}, we choose 
\begin{align}\label{calpha}
    c_{\alpha}\leq \frac{(1-(\frac{1+\rho}{2})^2)\sigma_{\text{min}}(D_0)}{4c_1\Vert D_{\sigma}\Vert \Vert B\Vert (1+\Vert A\Vert +\bar{K}\Vert B\Vert)(\bar{K}+1)\bar{\omega}}.
\end{align}
Hence, according to the update rule, we have
\begin{align}
    &\Vert K_{t+1}-K_{t}\Vert \nonumber\\
    = &\alpha_t\Vert (\text{smat}(\omega_{t+1})^{22}K_{t}-\text{smat}(\omega_{t+1})^{21})\Vert \nonumber \\
    \leq &\frac{c_{\alpha}}{(1+t)^\delta}(\bar{K}\Vert \text{smat}(\omega_{t+1})^{22}\Vert +\Vert \text{smat}(\omega_{t+1})^{21}\Vert )\nonumber \\
    \leq &\frac{c_{\alpha}}{(1+t)^\delta}(\bar{K}\Vert \omega_{t+1}\Vert +\Vert \omega_{t+1}\Vert)\nonumber \\
    \leq &\frac{c_{\alpha}}{(1+t)^\delta}(\bar{K}+1)\bar{\omega} \nonumber \\
    \leq &\frac{(1-(\frac{1+\rho}{2})^2)\sigma_{\text{min}}(D_0)}{4c_1\Vert D_{\sigma}\Vert \Vert B\Vert (1+\Vert A\Vert +\bar{K}\Vert B\Vert)}\frac{1}{(1+t)^\delta} \nonumber \\
    \leq &\frac{\sigma_{\text{min}}(D_0)}{4}\Vert D_{K_t}\Vert^{-1}\Vert B\Vert^{-1}(\Vert A-BK_t\Vert+1)^{-1},\label{ktdiff}
\end{align}
where the last inequality comes from \eqref{updk}.
Thus Lemma \ref{l2} holds for Algorithm \ref{alg1}. As a consequence, Proposition \ref{p1} is also guaranteed. Then for $I_2$, we get
\begin{align*}
I_2&=\sum\limits_{t=0}^{T-1}\frac{1}{\gamma_t}\mathbb{E}[y_t(J(K_t)-J(K_{t+1}))]\\
    &\leq \sum\limits_{t=0}^{T-1} \frac{l_1}{\gamma_t}\mathbb{E}[|y_t|\Vert K_{t+1}-K_t\Vert]\\
    &\leq \sum\limits_{t=0}^{T-1} \frac{l_1}{\gamma_t}\alpha_t(\bar{K}+1)\bar{\omega}\mathbb{E}[|y_t|]\\ 
    &\leq l_1(\bar{K}+1)\bar{\omega}(\sum\limits_{t=0}^{T-1} \mathbb{E}y^2_t)^{\frac{1}{2}}(\sum\limits_{t=0}^{T-1}\frac{\alpha^2_t}{\gamma^2_t})^{\frac{1}{2}}.
\end{align*}
For $I_3$, by the same argument it holds that
\begin{align*}
    I_3&=\sum\limits_{t=0}^{T-1}\frac{1}{\gamma_t}\mathbb{E}[(J(K_t)-J(K_{t+1}))^2]\\
    &\leq l^2_1(\bar{K}+1)^2\bar{{\omega}}^2\sum\limits_{t=0}^{T-1}\frac{1}{\gamma_t}\alpha_t^2\\
    &=\mathcal{O}(\sum\limits_{t=0}^{T-1}\frac{\alpha^2_t}{\gamma_t}).
\end{align*}
For $I_4$, we have
\begin{align*}
    I_4&=\sum\limits_{t=0}^{T-1}\gamma_t\mathbb{E}[(c_t-\eta_t)^2]\\
    &\leq \sum\limits_{t=0}^{T-1}\gamma_t\mathbb{E}[2c_t^2+2\eta_t^2]\\
    &\leq 2(C+U^2)\sum\limits_{t=0}^{T-1}\gamma_t.
\end{align*}

Since we have $\gamma_t=\frac{1}{(1+t)^v}$ and $ \alpha_t=\frac{c_{\alpha}}{(1+t)^\delta}$, where $0<v<\delta<1$, combining all terms together, we get
\begin{align}
    &\sum\limits_{t=0}^{T-1}\mathbb{E}[y^2_t] \nonumber\\
    \leq &\frac{2U^2}{\gamma_{T-1}}+l_1(\bar{K}+1)\bar{\omega}(\sum\limits_{t=0}^{T-1} \mathbb{E}y^2_t])^{\frac{1}{2}}(\sum\limits_{t=0}^{T-1}\frac{\alpha^2_t}{\gamma^2_t})^{\frac{1}{2}}\nonumber\\
    & +l^2_1(\bar{K}+1)^2\bar{{\omega}}^2\sum\limits_{t=0}^{T-1}\frac{1}{\gamma_t}\alpha_t^2+2(C+U^2)\sum\limits_{t=0}^{T-1}\gamma_t\nonumber \\
    \leq &2U^2T^v+l_1c^2_{\alpha}(\bar{K}+1)\bar{\omega}(\sum\limits_{t=0}^{T-1} \mathbb{E}y^2_t])^{\frac{1}{2}}\cdot \nonumber\\
    & (\sum\limits_{t=0}^{T-1}\frac{1}{(1+t)^{2(\delta-v)}})^{\frac{1}{2}} +4(C+U^2)\sum\limits_{t=0}^{T-1}\frac{1}{(1+t)^v}\nonumber \\
    \leq &2U^2T^v+l_1c^2_{\alpha}(\bar{K}+1)\bar{\omega}(\sum\limits_{t=0}^{T-1} \mathbb{E}y^2_t])^{\frac{1}{2}}\cdot \nonumber \\
    & (\frac{T^{1-2(\delta-v)}}{1-2(\delta-v)})^{\frac{1}{2}}+4(C+U^2)\frac{T^{1-v}}{1-v}. \label{eq97}
\end{align}
where the second inequality is due to
\begin{align*}
    l^2_1(\bar{K}+1)^2R^2_{\omega}\sum\limits_{t=0}^{T-1}\frac{1}{\gamma_t}\alpha_t^2\leq 2(C+U^2)\sum\limits_{t=0}^{T-1}\gamma_t.
\end{align*}
for large $T$ since $\delta>v$ and the last inequality is due to
\begin{align*}
    \sum\limits_{t=0}^{T-1}\frac{1}{(1+t)^v}\leq \int_0^{T}t^{-v}\,dt=\frac{T^{1-v}}{1-v}.
\end{align*}
Define
\begin{align*}
    X(T)&=\sum\limits_{t=0}^{T-1}\mathbb{E}[y^2_t],\\
    Y(T)&=\frac{T^{1-2(\delta-v)}}{1-2(\delta-v)},\\
    Z(T)&=2U^2T^v+4(C+U^2)\frac{T^{1-v}}{1-v}.
\end{align*}
Then from \eqref{eq97}, we get
\begin{align*}
    X(T)\leq Z(T)+l_1c^2_{\alpha}(\bar{K}+1)\bar{\omega}\sqrt{X(T)}\sqrt{Y(T)},
\end{align*}
which further gives
\begin{align}\label{eq17}
    (\sqrt{X(T)}&-\frac{l_1c^2_{\alpha}(\bar{K}+1)\bar{\omega}}{2}\sqrt{Y(T)})^2 \nonumber\\
    &\leq  Z(T)+(\frac{l_1c^2_{\alpha}(\bar{K}+1)\bar{\omega}}{2})^2Y(T).
\end{align}
Note that for a positive function $F(T)\leq G(T)+H(T)$, we have
\begin{align}
    F^2(T)&\leq 2G^2(T)+2H^2(T), \label{eq21}  \\
    \sqrt{F(T)}&\leq \sqrt{G(T)}+\sqrt{H(T)}. \label{eq23}
\end{align}
Hence, \eqref{eq17} implies
\begin{align*}
    \sqrt{X(T)}&-\frac{l_1c^2_{\alpha}(\bar{K}+1)\bar{\omega}}{2}\sqrt{Y(T)}\\
    & \leq \sqrt{Z(T)}+\frac{l_1c^2_{\alpha}(\bar{K}+1)\bar{\omega}}{2}\sqrt{Y(T)},\\
    \sqrt{X(T)}&\leq \sqrt{Z(T)}+l_1c^2_{\alpha}(\bar{K}+1)\bar{\omega}\sqrt{Y(T)},\\
    X(T)&\leq 2Z(T)+2l^2_1c^4_{\alpha}(\bar{K}+1)^2\bar{{\omega}}^2Y(T).
\end{align*}
Therefore, we have
\begin{align}
    \sum\limits_{t=0}^{T-1}\mathbb{E}[y^2_t]
    \leq\  &4U^2T^v+8(C+U^2)\frac{T^{1-v}}{1-v}\nonumber\\
    &+2l^2_1c^4_{\alpha}(\bar{K}+1)^2\bar{{\omega}}^2\frac{T^{1-2(\delta-v)}}{1-2(\delta-v)}\nonumber \\
    =\ &\mathcal{O}(T^v)+\mathcal{O}(T^{1-v})+\mathcal{O}(T^{1-2(\delta-v)}),\label{yt}
\end{align}
where by definition of $c_{\alpha}$ in \eqref{calpha}, $2l^2_1c^4_{\alpha}(\bar{K}+1)^2\Bar{\omega}^2$ is a small constant. Thus we finish the proof.
\end{proof}
\subsubsection{Approximating the Critic}
In this section, we show the convergence of critic. First, we need the following propositions.
\begin{proposition}\label{p2}
 For all the $K_t$, there exists a constant $\lambda>0$ such that
\begin{align*}
\sigma_{\text{min}}(A_{K_t})\ge\lambda.
\end{align*}
\end{proposition}
\begin{proposition}\label{p3}
(Lipschitz continuity of $\omega^\ast_t$) For any $\omega^\ast_t,\omega^\ast_{t+1}$, we have
\begin{align}
    \Vert\omega^\ast_t-\omega^\ast_{t+1}\Vert\leq l_2\Vert K_t-K_{t+1}\Vert, \label{eq19}
\end{align}
where
\begin{align}\label{newl2}
    l_2=\ &6c_1d^{\frac{3}{2}}\bar{K}(\Vert A\Vert+\Vert B\Vert)^2 \sigma_{\text{min}}^{-1}(D_0)\frac{\Vert D_{\sigma}\Vert\Vert R\Vert}{1-(\frac{1+\rho}{2})^2}\cdot \nonumber\\
    \ &(\bar{K}\Vert B\Vert(\Vert A\Vert +\bar{K}\Vert B\Vert+1) +1).
\end{align}
\end{proposition}
\noindent \textbf{Proof of Theorem \ref{t1}}:
\begin{proof}
Since we have $A_{K_t}\omega_t^\ast=b_{K_t}$, where $b_{K_t}=\mathbb{E}_{(x_t,u_t)}[(c(x_t,u_t)-J(K_t))\phi(x_t,u_t)]$, we can further get
\begin{align*}
    \Vert \omega_t^\ast\Vert &=\Vert A_{K_t}^{-1}b_{K_t}\Vert \\
    &\leq \frac{1}{\lambda}\Vert b_{K_t}\Vert\\
    &\leq \frac{1}{\lambda}\mathbb{E}\Vert c(x_t,u_t)-J(K_t)\Vert \Vert \phi(x_t,u_t)\Vert\\
    &\leq \frac{4}{\lambda}\mathbb{E}[c_t^2+J(K_t)^2+\Vert \phi(x_t,u_t)\Vert^2]\\
    &\leq \frac{4(C+U^2)}{\lambda},
\end{align*}
where the last inequality is due to Proposition \ref{new_bounded}.

Hence, we set
\begin{align}\label{rw}
    \bar{\omega}=\frac{4(C+U^2)}{\lambda}
\end{align}
such that all $\omega_t^\ast$ lie within this projection radius for all $t$.

From update rule of critic in Algorithm \ref{alg1}, we have
\begin{align*}
    \omega_{t+1}=\Pi_{\bar{\omega}}(\omega_{t} + \beta_t \delta_t \phi(x_{t},u_{t})),
\end{align*}
which further implies
\begin{align*}
    \omega_{t+1}-\omega_{t+1}^\ast = \Pi_{\bar{\omega}}(\omega_{t} + \beta_t \delta_t \phi(x_{t},u_{t}))-\omega_{t+1}^\ast.
\end{align*}
By applying 1-Lipschitz continuity of projection map, we have
\begin{align*}
    &\Vert\omega_{t+1}-\omega_{t+1}^\ast \Vert\\
    = &\Vert\Pi_{\bar{\omega}}(\omega_{t} + \beta_t \delta_t \phi(x_{t},u_{t}))-\omega_{t+1}^\ast\Vert\\
    =&\Vert\Pi_{\bar{\omega}}(\omega_{t} + \beta_t \delta_t \phi(x_{t},u_{t}))-\Pi_{\bar{\omega}}(\omega_{t+1}^\ast)\Vert\\
    \leq &\Vert \omega_{t} + \beta_t \delta_t \phi(x_{t},u_{t})-\omega_{t+1}^\ast\Vert\\
    =&\Vert \omega_t-\omega_t^\ast +\beta_t \delta_t \phi(s_{t},a_{t})+(\omega^\ast_t-\omega^\ast_{t+1})\Vert.
\end{align*}
This means
\begin{align*}
    &\Vert z_{t+1}\Vert^2\\
    \leq &\Vert z_t +\beta_t \delta_t \phi(s_{t},a_{t})+(\omega^\ast_t-\omega^\ast_{t+1})\Vert^2\\
=&\Vert z_t+\beta_t(g(O_t,\omega_t,K_t)+\Delta g(O_t,\eta_t,K_t))\\
&+(\omega^\ast_t-\omega^\ast_{t+1})\Vert^2\\ \nonumber
=&\Vert z_t\Vert^2+2\beta_t\langle z_t,g(O_t,\omega_t,K_t)\rangle\\
& +2\beta_t\langle z_t,\Delta g(O_t,\eta_t,K_t)\rangle +2\langle z_t,\omega^\ast_t-\omega^\ast_{t+1}\rangle\\
& +\Vert \beta_t(g(O_t,\omega_t,K_t)+\Delta g(O_t,\eta_t,K_t))+(\omega^\ast_t-\omega^\ast_{t+1})\Vert^2\\     \nonumber
=&\Vert z_t\Vert^2+2\beta_t\langle z_t,\bar{g}(\omega_t,K_t)\rangle+2\beta_t\Lambda(O_t,\omega_t,K_t)\\
& +2\beta_t\langle z_t,\Delta g(O_t,\eta_t,K_t)\rangle +2\langle z_t,\omega^\ast_t-\omega^\ast_{t+1}\rangle \\
&+\Vert \beta_t(g(O_t,\omega_t,K_t)+\Delta g(O_t,\eta_t,K_t))+(\omega^\ast_t-\omega^\ast_{t+1})\Vert^2\\ \nonumber
\leq &\Vert z_t\Vert^2+2\beta_t\langle z_t,\bar{g}(\omega_t,K_t)\rangle+2\beta_t\Lambda(O_t,\omega_t,K_t)\\
& +2\beta_t\langle z_t,\Delta g(O_t,\eta_t,K_t)\rangle +2\langle z_t,\omega^\ast_t-\omega^\ast_{t+1}\rangle\\
& +2\beta_t^2\Vert g(O_t,\omega_t,K_t)+\Delta g(O_t,\eta_t,K_t))\Vert^2\\
&+2\Vert \omega^\ast_t-\omega^\ast_{t+1} \Vert^2.
\end{align*}
From Proposition \ref{p2}, we know that $\sigma_{\text{min}}(A_{K_t})\ge\lambda$ for all $K_t$. 
Then we have
\begin{align*}
\langle z_t,\bar{g}(\omega_t,K_t)\rangle&=\langle z_t,b_{K_t}-A_{K_t}\omega_t\rangle\\
&=\langle z_t,b_{K_t}-A_{K_t}w_t-(b_{K_t}-A_{K_t}\omega^\ast_t)\rangle\\ 
&=\langle z_t,-A_{K_t}z_t\rangle\\
&=-z_t^\top A_{K_t}z_t\\
&\leq -\lambda\Vert z_t\Vert^2,
\end{align*}
where we use the fact $A_K\omega^\ast_{K_t}-b_{K_t}=0$. Hence, we have
\begin{align*}
    \Vert z_{t+1}\Vert^2&\leq (1-2\lambda\beta_t)\Vert z_t\Vert^2+2\beta_t\Lambda(O_t,\omega_t,K_t)\\
    &\quad +2\beta_t\langle z_t,\Delta g(O_t,\eta_t,K_t)\rangle+2\langle z_t,\omega^\ast_t-\omega^\ast_{t+1}\rangle\\
    &\quad +2\beta_t^2\Vert g(O_t,\omega_t,K_t)+\Delta g(O_t,\eta_t,K_t))\Vert^2\\
    &\quad +2\Vert \omega^\ast_t-\omega^\ast_{t+1} \Vert^2,
\end{align*}
Taking expectation up to $(x_{t},u_{t})$, it can be shown that
\begin{align}
    \mathbb{E}[\Vert z_{t+1}\Vert^2]
&\leq (1-2\lambda\beta_t)\mathbb{E}[\Vert z_t\Vert^2]\nonumber \\
&\quad +2\beta_t\mathbb{E}[\Lambda(O_t,\omega_t,K_t)]\nonumber \\
&\quad +2\beta_t\mathbb{E}[\langle z_t,\Delta g(O_t,\eta_t,K_t)\rangle] \nonumber\\
&\quad +2\mathbb{E}[\langle z_t,\omega^\ast_t-\omega^\ast_{t+1}\rangle]\nonumber \\
&\quad +2\mathbb{E}[\Vert \omega^\ast_t-\omega^\ast_{t+1} \Vert^2] \nonumber \\
&\quad +2\beta_t^2\mathbb{E}\Vert g(O_t,\omega_t,K_t)+\Delta g(O_t,\eta_t,K_t))\Vert^2. \label{eq18}
\end{align}
Similar to the previous argument, we have
\begin{align*}
    &\mathbb{E}[\Lambda(O_t,\omega_t,K_t)]\\
    =&\mathop{\mathbb{E}}\limits_{\vartheta_{0:t}}[\langle \omega_t-\omega^\ast_{K_t},g(O_t,\omega_t,K_t)-\bar{g}(\omega_t,K_t)\rangle]\\
    =&\mathop{\mathbb{E}}\limits_{\vartheta_{0:t-1}}\mathop{\mathbb{E}}\limits_{\vartheta_{0:t}}[\langle \omega_t-\omega^\ast_{K_t},g(O_t,\omega_t,K_t)-\bar{g}(\omega_t,K_t)\rangle|\vartheta_{0:t-1}]\\
    =&\mathop{\mathbb{E}}\limits_{\vartheta_{0:t-1}}\langle \omega_t-\omega^\ast_{K_t},\mathbb{E}_{\vartheta_{t}}[(g(O_t,\omega_t,K_t)-\bar{g}(\omega_t,K_t)|\vartheta_{0:t-1})]\rangle\\
    = &\ 0.
\end{align*}
For $\mathbb{E}\Vert g(O_t,\omega_t,K_t)+\Delta g(O_t,\eta_t,K_t))\Vert^2$, we have
\begin{align*}
    &\mathbb{E}\Vert g(O_t,\omega_t,K_t)+\Delta g(O_t,\eta_t,K_t))\Vert^2\\
    \leq&\  2\mathbb{E}\| (c_t-\eta_t)\phi(x_t,u_t)\|^2\\
    &\quad +2\mathbb{E}\|(\phi(x_t',u_t')-\phi(x_t,u_t))\phi(x_t,u_t)\|^2\|\omega_t\|^2.
\end{align*}
From Proposition \ref{new_bounded}, we know that $\mathbb{E}\| (c_t-\eta_t)\phi(x_t,u_t)\|^2$ is bounded. Based on the proof of Proposition \ref{new_bounded}, we know that $\|(\phi(x_t',u_t')-\phi(x_t,u_t))\phi(x_t,u_t)\|$ is the linear combination of the product of chi-square variables. From the fact that the expectation and variance of the product of chi-square variables are both bounded \cite[Corollary 5.4]{joarder2011statistical}, we know that $\mathbb{E}\|(\phi(x_t',u_t')-\phi(x_t,u_t))\phi(x_t,u_t)\|^2$ is also bounded. For simplicity, we set the constant $C$ large enough such that
\begin{align*}
    &\ \mathbb{E}\Vert g(O_t,\omega_t,K_t)+\Delta g(O_t,\eta_t,K_t))\Vert^2\\
    \leq&\  2\mathbb{E}\| (c_t-\eta_t)\phi(x_t,u_t)\|^2\\
    &\quad+2\mathbb{E}\|(\phi(x_t',u_t')-\phi(x_t,u_t))\phi(x_t,u_t)\|^2\|\omega_t\|^2\\
    \leq&\  2C^2+2\Bar{\omega}^2C^2\\
    \leq &\  2C^2(1+\Bar{\omega}^2).
\end{align*}
Therefore, we can further rewrite \eqref{eq18} as
\begin{align*}
    \mathbb{E}[\Vert z_{t+1}\Vert^2]
\leq &(1-2\lambda\beta_t)\mathbb{E}[\Vert z_t\Vert^2]\nonumber \\
&+2\beta_t\mathbb{E}[\langle z_t,\Delta g(O_t,\eta_t,K_t)\rangle] \nonumber\\
&+2\mathbb{E}[\langle z_t,\omega^\ast_t-\omega^\ast_{t+1}\rangle]+2\mathbb{E}[\Vert \omega^\ast_t-\omega^\ast_{t+1} \Vert^2] \nonumber \\
&+2\beta_t^2\mathbb{E}\Vert g(O_t,\omega_t,K_t)+\Delta g(O_t,\eta_t,K_t))\Vert^2\\
\leq & (1-2\lambda\beta_t)\mathbb{E}[\Vert z_t\Vert^2]+ 2\beta_t\sqrt{C}\mathbb{E}[\Vert z_t\Vert |y_t|]\\
&+2\mathbb{E}[\langle z_t,\omega^\ast_t-\omega^\ast_{t+1}\rangle]+2\mathbb{E}[\Vert \omega^\ast_t-\omega^\ast_{t+1} \Vert^2] \nonumber \\
&+4C^2(1+\bar{\omega}^2)\beta_t^2.
\end{align*}
Based on \eqref{eq19}, we can rewrite the above inequality as
\begin{align*}
    &\mathbb{E}[\Vert z_{t+1}\Vert^2]\\
    \leq &(1-2\lambda\beta_t)\mathbb{E}[\Vert z_t\Vert^2]+2l_2\mathbb{E}[\Vert z_t\Vert\Vert K_t-K_{t+1}\Vert] \nonumber \\
    &+2\sqrt{C}\beta_t\mathbb{E}[|y_t|\Vert z_t\Vert]+4C^2(1+\bar{\omega})^2\beta_t^2\\
    &+2l^2_2\mathbb{E}[\Vert K_t-K_{t+1} \Vert^2]\\
    \leq &(1-2\lambda\beta_t)\mathbb{E}[\Vert z_t\Vert^2]+2\sqrt{C}\beta_t\mathbb{E}[|y_t|\Vert z_t\Vert]\\
    &+2l_2c_3\frac{1}{(1+t)^\delta}\mathbb{E}[\Vert z_t\Vert] +2l^2_2c_3^2\frac{1}{(1+t)^{2\delta}}\nonumber \\
    &+4C^2(1+\bar{\omega})^2\beta_t^2\\
    \leq &(1-2\lambda\beta_t)\mathbb{E}[\Vert z_t\Vert^2]+2\sqrt{C}\beta_t\mathbb{E}[|y_t|\Vert z_t\Vert]\\
    &+2l_2c_3\frac{1}{(1+t)^\delta}\mathbb{E}[\Vert z_t\Vert] \nonumber \\
    & +(4C^2(1+\bar{\omega})^2+2l^2_2c_3^2)\frac{1}{(1+t)^{2v}},
\end{align*}
where the second inequality is due to $\Vert K_t-K_{t+1}\Vert\leq \frac{c_3}{(1+t)^\delta}$ from \eqref{ktdiff}, where
\begin{align}\label{c3}
    c_3:=\frac{(1-(\frac{1+\rho}{2})^2)\sigma_{\text{min}}(D_0)}{4c_1\Vert D_{\sigma}\Vert \Vert B\Vert (1+\Vert A\Vert +\bar{K}\Vert B\Vert)}.
\end{align}
Rearranging the inequality yields
\begin{align*}
    &2\lambda \mathbb{E}[\Vert z_t\Vert^2]\\
    \leq &\frac{1}{\beta_t}\mathbb{E}[(\Vert z_t\Vert^2-\Vert z_{t+1}\Vert^2)]+2\sqrt{C}\mathbb{E}[|y_t|\Vert z_t\Vert]\\
    &+2l_2c_3\frac{\mathbb{E}[\Vert z_t\Vert]}{(1+t)^{\delta-v}}\\
    & +(4C^2(1+\bar{\omega})^2+2l^2_2c_3^2)\frac{1}{(1+t)^v}.
\end{align*}
Summation from $0$ to $T-1$ gives
\begin{align*}
    &2\lambda \sum\limits_{t=0}^{T-1}\mathbb{E}[\Vert z_t\Vert^2]\\
    \leq &\underbrace{\sum\limits_{t=0}^{T-1}\frac{1}{\beta_t}\mathbb{E}[(\Vert z_t\Vert^2-\Vert z_{t+1}\Vert^2)]}_{I_1}\\
    &+2\sqrt{C}\underbrace{\sum\limits_{t=0}^{T-1}\mathbb{E}[|y_t|\Vert z_t\Vert]}_{I_2}\\
    & +2l_2c_3\underbrace{\sum\limits_{t=0}^{T-1}\frac{\mathbb{E}[\Vert z_t\Vert]}{(1+t)^{\delta-v}}}_{I_3}\\
    &+(4C^2(1+\bar{\omega})^2+2l^2_2c_3^2)\underbrace{\sum\limits_{t=0}^{T-1}\frac{1}{(1+t)^v}}_{I_4}.
\end{align*}
We need to control $I_1,I_2,I_3,I_4$ to approximate the critic.

For term $I_1$, from Abel summation by parts, we have
\begin{align*}
    I_1=&\sum\limits_{t=0}^{T-1}\frac{1}{\beta_t}\mathbb{E}[(\Vert z_t\Vert^2-\Vert z_{t+1}\Vert^2)]\\
    =&\sum\limits_{t=1}^{T-1}(\frac{1}{\beta_t}-\frac{1}{\beta_{t-1}})\mathbb{E}[\Vert z_t\Vert^2]+\frac{1}{\beta_{0}}\mathbb{E}[\Vert z_{0}\Vert^2]\\
    &-\frac{1}{\beta_{T-1}}\mathbb{E}[\Vert z_{T}\Vert^2]\\
    \leq &\sum\limits_{t=1}^{T-1}(\frac{1}{\beta_t}-\frac{1}{\beta_{t-1}})\mathbb{E}[\Vert z_t\Vert^2]+\frac{1}{\beta_{0}}\mathbb{E}[\Vert z_{0}\Vert^2]\\
    \leq &4\bar{\omega}^2(\sum\limits_{t=1}^{T-1}(\frac{
    1}{\beta_t}-\frac{1}{\beta_{t-1}})+\frac{1}{\beta_{0}})\\
    =&4\bar{\omega}^2\frac{1}{\beta_{T-1}}\\
    =&4\bar{\omega}^2T^v.
\end{align*}
For $I_2$, from Cauchy-Schwartz inequality, we have
\begin{align*}
    &\sum\limits_{t=0}^{T-1}\mathbb{E}[|y_t|\Vert z_t\Vert]\\
    \leq &
    \sum\limits_{t=0}^{T-1}(\mathbb{E}y_t^2)^{\frac{1}{2}}(\mathbb{E}\Vert z_t\Vert^2)^{\frac{1}{2}}\\
    \leq &(\sum\limits_{t=0}^{T-1}\mathbb{E}y^2_t)^{\frac{1}{2}}(\sum\limits_{t=0}^{T-1}\mathbb{E}\Vert z_t\Vert^2)^{\frac{1}{2}}\\
    = &(\mathcal{O}(T^v)+\mathcal{O}(T^{1-v})+\mathcal{O}(T^{1-2(\delta-v)}))^{\frac{1}{2}}(\sum\limits_{t=0}^{T-1}\mathbb{E}\Vert z_t\Vert^2)^{\frac{1}{2}}.
\end{align*}
where the last inequality comes from \eqref{yt}.

For $I_3$, we have
\begin{align*}
    \sum\limits_{t=0}^{T-1}\frac{\mathbb{E}[\Vert z_t\Vert]}{(1+t)^{\delta-v}} &\leq (\sum\limits_{t=0}^{T-1}\frac{1}{(1+t)^{2(\delta-v)}})^{\frac{1}{2}}(\sum\limits_{t=0}^{T-1}\mathbb{E}\Vert z_t\Vert^2)^{\frac{1}{2}}\\
    &\leq (\frac{T^{1-2(\delta-v)}}{1-2(\delta-v)})^{\frac{1}{2}}(\sum\limits_{t=0}^{T-1}\mathbb{E}\Vert z_t\Vert^2)^{\frac{1}{2}}.
\end{align*}

For $I_4$, we can bound it directly by
\begin{align*}
    \sum\limits_{t=0}^{T-1}\frac{1}{(1+t)^v}\leq \frac{T^{1-v}}{1-v}.
\end{align*}
Combining the upper bound of the above four items, we can get
\begin{align*}
    &2\lambda \sum\limits_{t=0}^{T-1}\mathbb{E}[\Vert z_t\Vert^2]\\
    \leq &(\mathcal{O}(T^v)+\mathcal{O}(T^{1-v})+\mathcal{O}(T^{1-2(\delta-v)}))^{\frac{1}{2}}(\sum\limits_{t=0}^{T-1}\mathbb{E}\Vert z_t\Vert^2)^{\frac{1}{2}}\\
    &+4\bar{\omega}^2T^v+2l_2c_3(\frac{T^{1-2(\delta-v)}}{1-2(\delta-v)})^{\frac{1}{2}}(\sum\limits_{t=0}^{T-1}\mathbb{E}\Vert z_t\Vert^2)^{\frac{1}{2}}\\
    &+(4C^2(1+\bar{\omega})^2+2l^2_2c_3^2)\frac{T^{1-v}}{1-v}\\
    =& (\mathcal{O}(T^v)+\mathcal{O}(T^{1-v})+\mathcal{O}(T^{1-2(\delta-v)}))^{\frac{1}{2}}(\sum\limits_{t=0}^{T-1}\mathbb{E}\Vert z_t\Vert^2)^{\frac{1}{2}}\\
    &+\mathcal{O}(T^v)+(\mathcal{O}(T^{1-2(\delta-v)}))^{\frac{1}{2}}(\sum\limits_{t=0}^{T-1}\mathbb{E}\Vert z_t\Vert^2)^{\frac{1}{2}}\\
    &+\mathcal{O}(T^{1-v}).
\end{align*}
Define
\begin{align*}
    A(T)=&\sum\limits_{t=0}^{T-1} \mathbb{E}[\Vert  z_t\Vert^2],\\
    B(T)=&\mathcal{O}(T^v)+\mathcal{O}(T^{1-v}),\\
    C(T)=&\mathcal{O}(T^{1-2(\delta-v)}),\\
    D(T)=&\mathcal{O}(T^v)+\mathcal{O}(T^{1-v})+\mathcal{O}(T^{1-2(\delta-v)}).
\end{align*}
So we have
\begin{align*}
    2\lambda A(T)\leq B(T)+\sqrt{C(T)}\sqrt{A(T)}+\sqrt{D(T)}\sqrt{A(T)}.
\end{align*}
Thus we get
\begin{align*}
    &2\lambda [\sqrt{A(T)}-(\frac{1}{4\lambda}\sqrt{C(T)}+\frac{1}{4\lambda}\sqrt{D(T)})]^2\\
    \leq &B(T)+2\lambda (\frac{1}{4\lambda}\sqrt{C(T)}+\frac{1}{4\lambda}\sqrt{D(T)})^2,
\end{align*}
which implies that
\begin{align}
    &[\sqrt{A(T)}-(\frac{1}{4\lambda}\sqrt{C(T)}+\frac{1}{4\lambda}\sqrt{D(T)})]^2\nonumber\\
    \leq &\frac{B(T)}{2\lambda}+(\frac{1}{4\lambda}\sqrt{C(T)}+\frac{1}{4\lambda}\sqrt{D(T)})^2,\nonumber \\
    &\sqrt{A(T)}-(\frac{1}{4\lambda}\sqrt{C(T)}+\frac{1}{4\lambda}\sqrt{D(T)})\nonumber \\
    \leq &\sqrt{\frac{B(T)}{2\lambda}}+(\frac{1}{4\lambda}\sqrt{C(T)}+\frac{1}{4\lambda}\sqrt{D(T)}),\nonumber \\
    &\sqrt{A(T)}\leq \sqrt{\frac{B(T)}{2\lambda}}+(\frac{1}{2\lambda}\sqrt{C(T)}+\frac{1}{2\lambda}\sqrt{D(T)})\nonumber, \\
    &A(T)\leq \frac{B(T)}{\lambda}+\frac{1}{2\lambda^2}(\sqrt{C(T)}+\sqrt{D(T)})^2\nonumber, \\
    &A(T)\leq \frac{B(T)}{\lambda}+\frac{1}{\lambda^2}(C(T)+D(T)). \label{eq20}
\end{align}
Consequently, \eqref{eq20} gives
\begin{align*}
    \sum\limits_{t=0}^{T-1} \mathbb{E}[\Vert  z_t\Vert^2]\leq &\mathcal{O}(T^v)+\mathcal{O}(T^{1-v})+\mathcal{O}(T^{1-2(\delta-v)}).
\end{align*}
Hence, we can get
\begin{align}\label{ccritic}
    \frac{1}{T}\sum\limits_{t=0}^{T-1}\mathbb{E}[\Vert \omega_t-\omega_t^\ast\Vert^2]\leq &\mathcal{O}(\frac{1}{T^{1-v}})+\mathcal{O}(\frac{1}{T^v})\nonumber \\
    &+\mathcal{O}(\frac{1}{T^{2(\delta-v)}}),
\end{align}
which concludes the proof.
\end{proof}
\subsection{Convergence of actor}
To prove the convergence of actor, we need the following two lemmas, which characterize two important properties of LQR system.
\begin{lemma}\label{l6}
(Almost Smoothness). For any two stable policies $K$ and $K'$, $J(K)$ and $J(K')$ satisfy:
\begin{align*}
    J(K')&-J(K)=-2\text{Tr}(D_{K'}(K-K')^\top E_K)\\
    &+\text{Tr}(D_{K'}(K-K')^\top(R+B^\top P_KB)(K-K')).
\end{align*}
\end{lemma}
\begin{lemma}\label{lem:l7}
(Gradient Domination). Let $K^\ast$ be an optimal policy. Suppose $K$ has finite cost. Then, it holds that
\begin{align*}
    \sigma_{\text{min}}(D_0)\Vert R+&B^\top P_K B\Vert^{-1}\text{Tr}(E_K^\top E_K)\leq J(K)-J(K^\ast)\\
    \leq &\frac{1}{\sigma_{\text{min}}(R)}\Vert D_{K^\ast}\Vert \text{Tr}(E_K^\top E_K).
\end{align*}
\end{lemma}
\noindent \textbf{Proof of Theorem \ref{t2}}:
\begin{proof}
From the update rule of actor, we know that
\begin{align*}
    K_{t+1}=K_t-\alpha_t\widehat{\nabla_{K_t}^N J(K_t)},
\end{align*}
where $\widehat{\nabla_{K_t}^N J(K_t)}=\text{smat}(\omega_{t+1})^{22}K_{t}-\text{smat}(\omega_{t+1})^{21}$. We define $\hat{E}_{K_t}=\widehat{\nabla_{K_t}^N J(K_t)}$ for simplicity.  Combining the almost smoothness property, we get
\begin{align*}
    &J(K_{t+1})-J(K_t)\\
    =&-2\text{Tr}(D_{K_{t+1}}(K_t-K_{t+1})^\top E_{K_t})+\nonumber\\
    &\text{Tr}(D_{K_{t+1}}(K_t-K_{t+1})^\top(R+B^\top P_{K_t} B)(K_t-K_{t+1}))\\
    =&-2\alpha_t\text{Tr}(D_{K_{t+1}}\hat{E}_{K_t}^\top E_{K_t})\\
    &+\alpha_t^2\text{Tr}(D_{K_{t+1}}\hat{E}^\top_{K_t}(R+B^\top P_{K_t} B)\hat{E}_{K_t})\\
    =&-2\alpha_t\text{Tr}(D_{K_{t+1}}(\hat{E}_{K_t}-E_{K_t})^\top E_{K_t})\\
    &-2\alpha_t\text{Tr}(D_{K_{t+1}}E_{K_t}^\top E_{K_t})\nonumber\\
    & +\alpha_t^2\text{Tr}(D_{K_{t+1}}\hat{E}^\top_{K_t}(R+B^\top P_{K_t} B)\hat{E}_{K_t}).\\
\end{align*}
By the similar trick to the proof of Proposition \ref{pa1.2.1}, we can bound $P_{K_t}$ by
\begin{align*}
    \Vert P_{K_t}\Vert\leq &\frac{\hat{c}_1}{1-(\frac{1+\rho}{2})^2}\Vert Q+K^\top RK\Vert\\
    \leq &\frac{\hat{c}_1(\sigma_{\text{max}}(Q)+\bar{K}^2\sigma_{\text{max}}(R))}{1-(\frac{1+\rho}{2})^2},
\end{align*}
where $\hat{c}_1$ is a constant. Hence we further have
\begin{align*}
    &\text{Tr}(D_{K_{t+1}}\hat{E}^\top_{K_t}(R+B^\top P_{K_t} B)\hat{E}_{K_t})\\
    \leq &d\Vert D_{K_{t+1}}\Vert \Vert R+B^\top P_{K_t}B\Vert \Vert \hat{E}_{K_t}\Vert^2_{\text{F}}\\
    \leq &d(\bar{K}+1)^2\bar{\omega}^2 \frac{c_1\Vert D_{\sigma}\Vert}{1-(\frac{1+\rho}{2})^2}(\sigma_{\text{max}}(R)\\
    &+\sigma^2_{\text{max}}(B)\frac{\hat{c}_1(\sigma_{\text{max}}(Q)+\bar{K}^2\sigma_{\text{max}}(R))}{1-(\frac{1+\rho}{2})^2}),
\end{align*}
where we use $\Vert \hat{E}_{K_t}\Vert_{\text{F}}\leq (\bar{K}+1)\bar{\omega}$. Hence we define $c_4$ as follows
\begin{align}\label{c4}
    c_4:= &d^2(\bar{K}+1)^2\bar{\omega}^2 \frac{c_1\Vert D_{\sigma}\Vert}{1-(\frac{1+\rho}{2})^2}(\sigma_{\text{max}}(R)\nonumber \\
    &+\sigma^2_{\text{max}}(B)\frac{\hat{c}_1(\sigma_{\text{max}}(Q)+\bar{K}^2\sigma_{\text{max}}(R))}{1-(\frac{1+\rho}{2})^2}).
\end{align}
Then we get
\begin{align*}
    &J(K_{t+1})-J(K_t)\\
    \leq &-2\alpha_t\text{Tr}(D_{K_{t+1}}(\hat{E}_{K_t}-E_{K_t})^\top E_{K_t})\\
    &-2\alpha_t\text{Tr}(D_{K_{t+1}}E_{K_t}^\top E_{K_t})+c_4\alpha_t^2\nonumber \\
    \leq &\alpha_t\frac{2c_1d^{\frac{3}{2}}\Vert D_{\sigma}\Vert}{1-(\frac{1+\rho}{2})^2}\Vert E_{K_t}\Vert \Vert \hat{E}_{K_t}-E_{K_t}\Vert\\
    &-2\alpha_t\sigma_{\text{min}}(D_0)\Vert E_{K_t}\Vert^2+c_4\alpha_t^2\\
    =&c_5\alpha_t\Vert E_{K_t}\Vert \Vert \hat{E}_{K_t}-E_{K_t}\Vert-2\alpha_t\sigma_{\text{min}}(D_0)\Vert E_{K_t}\Vert^2+c_4\alpha_t^2,
\end{align*}
where
\begin{align}\label{c_5}
    c_5:=\frac{2c_1d^{\frac{3}{2}}\Vert D_{\sigma}\Vert}{1-(\frac{1+\rho}{2})^2}.
\end{align}
Considering the restriction that $\bar{A}$ holds and taking expectation up to $(x_{t},u_{t})$, we have
\begin{align*}
    \mathbb{E}[\Vert E_{K_t}\Vert^2]&\leq \frac{\mathbb{E}[(J(K_t)-J(K_{t+1}))]}{2\alpha_t\sigma_{\text{min}}(D_0)}\\
    &+\frac{c_5}{2\sigma_{\text{min}}(D_0)}\mathbb{E}[\Vert E_{K_t}\Vert \Vert \hat{E}_{K_t}-E_{K_t}\Vert]\\
    &+\frac{c_4\alpha_t}{2\sigma_{\text{min}}(D_0)}.
\end{align*}
Summing over $t$ from $0$ to $T-1$ gives
\begin{align*}
    \sum\limits_{t=0}^{T-1}\mathbb{E} [\Vert E_{K_t}\Vert^2]&\leq \underbrace{\sum\limits_{t=0}^{T-1} \frac{\mathbb{E}[(J(K_t)-J(K_{t+1}))]}{2\alpha_t\sigma_{\text{min}}(D_0)}}_{I_1}\\
    &+\frac{c_5}{2\sigma_{\text{min}}(D_0))}\underbrace{\sum\limits_{t=0}^{T-1} \mathbb{E}[\Vert E_{K_t}\Vert \Vert \hat{E}_{K_t}-E_{K_t}\Vert]}_{I_2}\nonumber\\
    &+\frac{c_4}{2\sigma_{\text{min}}(D_0)}\underbrace{\sum\limits_{t=0}^{T-1}\alpha_t}_{I_3}.
\end{align*}
For term $I_1$, using Abel summation by parts, we have
\begin{align*}
    &\sum\limits_{t=0}^{T-1} \frac{\mathbb{E}[(J(K_t)-J(K_{t+1}))]}{2\alpha_t\sigma_{\text{min}}(D_0)}\\
    =&\frac{1}{2\sigma_{\text{min}}(D_0)}(\sum\limits_{t=1}^{T-1} (\frac{1}{\alpha_{t}}-\frac{1}{\alpha_{t-1}})\mathbb{E}[J(K_t)] \nonumber \\
    & +\frac{1}{\alpha_{0}}\mathbb{E}[J(K_{0})]-\frac{1}{\alpha_{T-1}}\mathbb{E}[J(K_{T})])\\
    \leq &\frac{U}{2\sigma_{\text{min}}(D_0)}(\sum\limits_{t=1}^{T-1} (\frac{1}{\alpha_{t}}-\frac{1}{\alpha_{t-1}})+\frac{1}{\alpha_{0}})\\
    =&\frac{U}{2\sigma_{\text{min}}(D_0)}\frac{1}{\alpha_{T-1}}.
\end{align*}
For term $I_2$, by Cauchy-Schwartz inequality, we have
\begin{align*}
    &\sum\limits_{t=0}^{T-1}\mathbb{E}[\Vert E_{K_t}\Vert \Vert \hat{E}_{K_t}-E_{K_t}\Vert]\\
    \leq &(\sum\limits_{t=0}^{T-1}\mathbb{E}\Vert E_{K_t}\Vert^2)^{\frac{1}{2}}(\sum\limits_{t=0}^{T-1}\mathbb{E}\Vert \hat{E}_{K_t}-E_{K_t}\Vert^2)^{\frac{1}{2}}.
\end{align*}
For term $I_3$, following the previous trick, we have
\begin{align*}
    \sum\limits_{t=0}^{T-1}\alpha_t\leq \frac{c_{\alpha}}{1-\delta}T^{1-\delta}.
\end{align*}
Combining the results of $I_1, I_2$ and $I_3$, we have
\begin{align*}
    &\sum\limits_{t=0}^{T-1} \mathbb{E}[\Vert E_{K_t}\Vert^2]\\
    \leq &\frac{U}{2\sigma_{\text{min}}(D_0)c_{\alpha}}T^{\delta}+\frac{c_4c_{\alpha}}{2\sigma_{\text{min}}(D_0)(1-\delta)}T^{1-\delta}\nonumber \\
    +&\frac{c_5}{2\sigma_{\text{min}}(D_0)}(\sum\limits_{t=0}^{T-1}\mathbb{E}\Vert E_{K_t}\Vert^2)^{\frac{1}{2}}(\sum\limits_{t=0}^{T-1}\mathbb{E} \Vert \hat{E}_{K_t}-E_{K_t}\Vert^2)^{\frac{1}{2}}.
\end{align*}
Dividing $T$ at both sides, we can express the result as
\begin{align*}
    &\frac{1}{T}\sum\limits_{t=0}^{T-1} \mathbb{E}[\Vert E_{K_t}\Vert^2]\\
    \leq &\frac{U}{2\sigma_{\text{min}}(D_0)c_{\alpha}}\frac{1}{T^{1-\delta}}+\frac{c_4c_{\alpha}}{2\sigma_{\text{min}}(D_0)(1-\delta)}\frac{1}{T^{\delta}}\nonumber \\
    &+\frac{c_5}{2\sigma_{\text{min}}(D_0)}\frac{1}{T}(\sum\limits_{t=0}^{T-1}\mathbb{E}\Vert E_{K_t}\Vert^2)^{\frac{1}{2}}\cdot \\
    &(\sum\limits_{t=0}^{T-1} \mathbb{E}\Vert \hat{E}_{K_t}-E_{K_t}\Vert^2)^{\frac{1}{2}}.\nonumber
\end{align*}
Denote
\begin{align*}
    F(T)&=\frac{1}{T}\sum\limits_{t=0}^{T-1}\mathbb{E} [\Vert E_{K_t}\Vert^2],\\
    H(T)&=\frac{1}{T}\sum\limits_{t=0}^{T-1}\mathbb{E} [\Vert \hat{E}_{K_t}-E_{K_t}\Vert^2].
\end{align*}
Then we have
\begin{align*}
    F(T)\leq &\mathcal{O}(\frac{1}{T^{1-\delta}})+\mathcal{O}(\frac{1}{T^{\delta}})\\
    &+ \frac{c_5}{2\sigma_{\text{min}}(D_0)}\sqrt{F(T)} \sqrt{H(T)},
\end{align*}
which further implies
\begin{align*}
    &(\sqrt{F(T)}-\frac{c_5}{4\sigma_{\text{min}}(D_0)}\sqrt{H(T)})^2\\
    \leq &\mathcal{O}(\frac{1}{T^{1-\delta}})+\mathcal{O}(\frac{1}{T^{\delta}})+(\frac{c_5}{4\sigma_{\text{min}}(D_0)})^2H(T).
\end{align*}
From \eqref{eq21} and \eqref{eq23}, we know that
\begin{align*}
    &\sqrt{F(T)}-\frac{c_5}{4\sigma_{\text{min}}(D_0)}\sqrt{H(T)}\\
    \leq &\sqrt{\mathcal{O}(\frac{1}{T^{1-\delta}})+\mathcal{O}(\frac{1}{T^{\delta}})}+\frac{c_5}{4\sigma_{\text{min}}(D_0)}\sqrt{H(T)},\\
    \sqrt{F(T)}&\leq \sqrt{\mathcal{O}(\frac{1}{T^{1-\delta}})+\mathcal{O}(\frac{1}{T^{\delta}})}+\frac{c_5}{2\sigma_{\text{min}}(D_0)}\sqrt{H(T)},\\
    F(T)&\leq \mathcal{O}(\frac{1}{T^{1-\delta}})+\mathcal{O}(\frac{1}{T^{\delta}})+(\frac{c_5}{\sigma_{\text{min}}(D_0)})^2H(T).
\end{align*}
Since we have
\begin{align*}
    &\hat{E}_{K_t}=\text{smat}(\omega_{t+1})^{22}K_{t}-\text{smat}(\omega_{t+1})^{21},\\  &E_{K_t}=\text{smat}(\omega^\ast_{t+1})^{22}K_{t}-\text{smat}(\omega^\ast_{t+1})^{21}.
\end{align*}
Thus we get
\begin{align*}
    &\sum\limits_{t=0}^{T-1}\mathbb{E}[\Vert \hat{E}_{K_t}-E_{K_t}\Vert^2]\\
    =&\sum\limits_{t=0}^{T-1}\mathbb{E}[\Vert (\text{smat}(\omega_{t+1})^{22}-\text{smat}(\omega^\ast_{t+1})^{22})K_{t}\\
    &-(\text{smat}(\omega_{t+1})^{21}-\text{smat}(\omega^\ast_{t+1})^{21})\Vert^2]\\
    \leq &\sum\limits_{t=0}^{T-1} 2\mathbb{E}[\Vert (\text{smat}(\omega_{t+1})^{22}-\text{smat}(\omega^\ast_{t+1})^{22})K_{t}\Vert^2] \nonumber\\
    & +\sum\limits_{t=0}^{T-1}[2\mathbb{E}\Vert \text{smat}(\omega_{t+1})^{21}-\text{smat}(\omega^\ast_{t+1})^{21}\Vert^2]\\
    \leq &\sum\limits_{t=0}^{T-1}2(\bar{K}+1)\mathbb{E}[\Vert \text{smat}(\omega_{t+1})-\text{smat}(\omega^\ast_{t+1})\Vert^2_{\text{F}}]\\
    =&\sum\limits_{t=0}^{T-1}2(\bar{K}+1)\mathbb{E}[\Vert \omega_{t+1}-\omega^\ast_{t+1}\Vert^2].
\end{align*}
From the convergence of critic in \eqref{ccritic}, we have
\begin{align*}
    H(T)=&\frac{1}{T}\sum\limits_{t=0}^{T-1} \mathbb{E}[\Vert \hat{E}_{K_t}-E_{K_t}\Vert^2] \\
    \leq &\mathcal{O}(\frac{1}{T^{1-v}})+\mathcal{O}(\frac{1}{T^v})+\mathcal{O}(\frac{1}{T^{2(\delta-v)}}).
\end{align*}
Overall, we get 
\begin{align}\label{n1}
    \frac{1}{T}&\sum\limits_{t=0}^{T-1} \mathbb{E}[\Vert E_{K_t}\Vert^2] \nonumber \\ \leq &\mathcal{O}(\frac{1}{T^{1-\delta}})+\mathcal{O}(\frac{1}{T^v})+\mathcal{O}(\frac{1}{T^{2(\delta-v)}}).
\end{align}
Since we have $0<v<\sigma<1$, the fastest convergence rate is
\begin{align*}
    \text{maxmin}\{1-\delta, v, 2(\delta-v)\}=\frac{2}{5},
\end{align*}
where the optimal value is attained by choosing $\delta=\frac{3}{5}, v=\frac{2}{5}$. From gradient domination, we know that
\begin{align}\label{n2}
    \mathbb{E}[(J(K_t)-J(K^\ast))]\leq &\frac{1}{\sigma_{\text{min}}(R)}\Vert D_{K^\ast}\Vert \mathbb{E}[\text{Tr}(E_{K_t}^\top E_{K_t})]\nonumber \\
    \leq& \frac{d\Vert D_{K^\ast}\Vert}{\sigma_{\text{min}}(R)} \mathbb{E}[\Vert E_{K_t}\Vert^2].
\end{align}
Hence, combining \eqref{n1}, we have
\begin{align*}
    \mathop{\text{min}}\limits_{0\leq t< T}&\frac{d\Vert D_{K^\ast}\Vert}{\sigma_{\text{min}}(R)} \mathbb{E}[\Vert E_{K_t}\Vert^2]\\
    &\leq \mathcal{O}(\frac{1}{T^{1-\delta}})+\mathcal{O}(\frac{1}{T^v})+\mathcal{O}(\frac{1}{T^{2(\delta-v)}}).  
\end{align*}
Therefore, we get
\begin{align*}
    \mathop{\text{min}}\limits_{0\leq t< T}&\mathbb{E}[(J(K_t)-J(K^\ast))]\\
    &\leq \mathcal{O}(\frac{1}{T^{1-\delta}})+\mathcal{O}(\frac{1}{T^v})+\mathcal{O}(\frac{1}{T^{2(\delta-v)}}).
\end{align*}
Thus we conclude the convergence of actor.
\end{proof}
\section{Proof of Propositions}\label{appendix2}
To establish the Proposition \ref{pro3}, we need the following lemma, the proof of which can be found in \citet{nagar1959bias,magnus1978moments}.
\begin{lemma}\label{lemmab1}
Let $g\sim\mathcal{N}(0,I_n)$ be the standard Gaussian random variable in $\mathbb{R}^n$ and let $M,N$ be two symmetric matrices. Then we have
\begin{align*}
    \mathbb{E}[g^\top Mgg^\top N g]=2\text{Tr}(MN)+\text{Tr}(M)\text{Tr}(N).
\end{align*}
\end{lemma}
\noindent \textbf{Proof of Proposition \ref{pro3}}:

\begin{proof}
This proposition is a slight modification of lemma 3.2 in \citet{yang2019provably} and the proof is inspired by the proof of this lemma.


For any state-action pair $(x^\top,u^\top)^\top\in\mathbb{R}^{d+k}$, we denote the successor state-action pair following policy $\pi_K$ by $(x'^\top,u'^\top)^\top$. With this notation, as we defined in \eqref{policy}, we have
\begin{align*}
    x'=Ax+Bu+\epsilon, \quad u'= -Kx'+\sigma\zeta.
\end{align*}
where $\epsilon\sim\mathcal{N}(0,D_0)$ and $\zeta\sim\mathcal{N}(0,I_k)$. We further denote $(x^\top,u^\top)^\top$ and $(x'^\top,u'^\top)^\top$ by $\vartheta$ and $\vartheta'$ respectively. Therefore, we have
\begin{align}\label{markov2}
    \vartheta'=L\vartheta+\varepsilon,
\end{align}
where
\begin{align*}
     L:=&\begin{bmatrix}
     A & B\\ -KA & -KB
    \end{bmatrix}=\begin{bmatrix}
     I_d \\ -K
    \end{bmatrix}\begin{bmatrix}
     A & B
    \end{bmatrix}, \\ 
    \varepsilon:=&\begin{bmatrix}
    \epsilon \\ -K\epsilon+\sigma\zeta
    \end{bmatrix}.
\end{align*}
Therefore, by definition, we have $\varepsilon\sim\mathcal{N}(0,\tilde{D}_0)$ where
\begin{align*}
    \tilde{D}_0=\begin{bmatrix}
    D_0 & -D_0 K^\top \\ -KD_0 & KD_0 K^\top+\sigma^2I_k
    \end{bmatrix}.
\end{align*}
Since for any two matrices $M$ and $N$, it holds that $\rho(MN)=\rho(NM)$. Then we get $\rho(L)=\rho(A-BK)<1$. Consequently, the Markov chain defined in \eqref{markov2} have a stationary distribution $\mathcal{N}(0,\tilde{D}_K)$ denoted by $\tilde{\rho}_K$, where $\tilde{D}_K$ is the unique positive definite solution of the following Lyapunov equation
\begin{align}\label{lyap2}
    \tilde{D}_K=L\tilde{D}_KL^\top + \tilde{D}_0.
\end{align}
Meanwhile, from the fact that $x\sim\mathcal{N}(0,D_K)$ and $u=-Kx+\sigma\zeta$, by direct computation we have
\begin{align*}
    \tilde{D}_K&=\begin{bmatrix}
    D_K & -D_KK^\top \\ -KD_K & KD_KK^\top+\sigma^2I_k
    \end{bmatrix}\\
    &=\begin{bmatrix}
    0 & 0 \\ 0 & \sigma^2I_k
    \end{bmatrix}+\begin{bmatrix}
    I_d \\ -K
    \end{bmatrix}D_K\begin{bmatrix}
    I_d \\ -K
    \end{bmatrix}^\top.
\end{align*}
From the fact that $\Vert AB\Vert_{\text{F}}\leq \Vert A\Vert_{\text{F}}\Vert B\Vert$ and $\Vert A\Vert\leq \Vert A\Vert_{\text{F}}$, we have
\begin{align}\label{tilde_D_K}
    \Vert\tilde{D}_K\Vert\leq \Vert \tilde{D}_K\Vert_{\text{F}}\leq \sigma^2k+\Vert D_{K}\Vert(d+\Vert K\Vert^2_{\text{F}}).
\end{align}
Then we get
\begin{align*}
    \mathbb{E}_{(x,u)}[\phi(x,u)\phi(x,u)^\top]=\mathbb{E}_{\vartheta\sim\tilde{\rho}_K}[\phi(\vartheta)\phi(\vartheta)^\top].
\end{align*}
Let $M,N$ be any two symmetric matrices with appropriate dimension, we have
\begin{align*}
    &\quad \ \text{svec}(M)^\top \mathbb{E}_{\vartheta\sim\tilde{\rho}_K}[\phi(\vartheta)\phi(\vartheta)^\top]\text{svec}(N)\\
    &=\mathbb{E}_{\vartheta\sim\tilde{\rho}_K}[\text{svec}(M)^\top \phi(\vartheta)\phi(\vartheta)^\top \text{svec}(N)]\\
    &=\mathbb{E}_{\vartheta\sim\tilde{\rho}_K}[\langle \vartheta\vartheta^\top, M \rangle \langle \vartheta\vartheta^\top,N\rangle]\\
    &=\mathbb{E}_{\vartheta\sim\tilde{\rho}_K}[\vartheta^\top M\vartheta \vartheta^\top N\vartheta]\\
    &=\mathbb{E}_{g\sim\mathcal{N}(0,I_{d+k})}[g^\top \tilde{D}_K^{1/2} M\tilde{D}_K^{1/2}gg^\top\tilde{D}_K^{1/2}N\tilde{D}_K^{1/2}g],
\end{align*}
where $\tilde{D}_K^{1/2}$ is the square root of $\tilde{D}_K$. By applying Lemma \ref{lemmab1}, we have
\begin{align*}
    &\text{svec}(M)^\top \mathbb{E}_{\vartheta\sim\tilde{\rho}_K}[\phi(\vartheta)\phi(\vartheta)^\top]\text{svec}(N)\\
    =&\mathbb{E}_{g\sim\mathcal{N}(0,I_{d+k})}[g^\top \tilde{D}_K^{1/2} M\tilde{D}_K^{1/2}gg^\top\tilde{D}_K^{1/2}N\tilde{D}_K^{1/2}g]\\
    =&2\text{Tr}(\tilde{D}_K^{1/2}M\tilde{D}_KN\tilde{D}_K^{1/2})\\
    &+\text{Tr}(\tilde{D}_K^{1/2}M\tilde{D}_K^{1/2})\text{Tr}(\tilde{D}_K^{1/2}N\tilde{D}_K^{1/2})\\
    =&2\langle M,\tilde{D}_KN\tilde{D}_K\rangle+\langle M,\tilde{D}_K\rangle \langle N,\tilde{D}_K\rangle\\
    =&\text{svec}(M)^\top (2\tilde{D}_K\otimes_s \tilde{D}_K+\text{svec}(\tilde{D}_K)\text{svec}(\tilde{D}_K)^\top)\cdot\\
    &\text{svec}(N),
\end{align*}
where the last equality follows from the fact that 
\begin{align*}
    \text{svec}(\frac{1}{2}(NSM^\top+MSN^\top))=(M\otimes_s N)\text{svec}(S).
\end{align*}
for any two matrix $M,N$ and a symmetric matrix $S$ \cite{schacke2004kronecker}. Thus we have
\begin{align}\label{eq226}
    &\mathbb{E}_{\vartheta\sim\tilde{\rho}_K}[\phi(\vartheta)\phi(\vartheta)^\top] \nonumber\\
    =&2\tilde{D}_K\otimes_s \tilde{D}_K+\text{svec}(\tilde{D}_K)\text{svec}(\tilde{D}_K)^\top.
\end{align}
Similarly
\begin{align*}
    \phi(\vartheta')&=\text{svec}[(L\vartheta+\varepsilon)(L\vartheta+\varepsilon)^\top]\\
    &=\text{svec}(L\vartheta\vartheta^\top L^\top+L\vartheta\varepsilon^\top-\varepsilon \vartheta^\top L^\top+\varepsilon\varepsilon^\top).
\end{align*}
Since $\epsilon$ is independent of $\vartheta$, we get
\begin{align*}
    \mathbb{E}_{\vartheta\sim\tilde{\rho}_K}&[\phi(\vartheta)\phi(\vartheta')^\top]\\
    &=\mathbb{E}_{\vartheta\sim\tilde{\rho}_K}[\phi(\vartheta)\text{svec}(L\vartheta\vartheta^\top L^\top+\tilde{D}_0)].
\end{align*}
By the same argument, we have
\begin{align*}
    &\text{svec}(M)^\top \mathbb{E}_{\vartheta\sim\tilde{\rho}_K}[\phi(\vartheta)\phi(\vartheta')^\top] \text{svec}(N)\\
    =&\mathbb{E}_{\vartheta\sim\tilde{\rho}_K} [\langle \vartheta\vartheta^\top,M\rangle \langle L\vartheta\vartheta^\top L^\top+\tilde{D}_0,N\rangle]\\
    =&\mathbb{E}_{\vartheta\sim\tilde{\rho}_K}[\vartheta^\top M\vartheta\vartheta^\top L^\top NL\vartheta] +\langle M,\tilde{D}_K\rangle \langle \tilde{D}_0,N\rangle]\\
    =&\mathbb{E}_{g\in\mathcal{N}(0,I_{d+k})}[g^\top\tilde{D}_K^{\frac{1}{2}}M\tilde{D}_K^{\frac{1}{2}}gg^\top\tilde{D}_K^{\frac{1}{2}}L^\top NL\tilde{D}_K^{\frac{1}{2}}g]\\
    &+\langle M,\tilde{D}_K,\rangle \langle \tilde{D}_0,N\rangle]\\
    =&2\text{Tr}(M\tilde{D}_KL^\top NL\tilde{D}_K)+\text{Tr}(M\tilde{D}_K)\text{Tr}(L^\top NL \tilde{D}_K)\\
    &+\langle M,\tilde{D}_K\rangle \langle \tilde{D}_0,N\rangle\\
    =&2\langle M,\tilde{D}_KL^\top NL\tilde{D}_K\rangle+\langle M,\tilde{D}_K\rangle \langle L \tilde{D}_KL^\top, N\rangle \\
    &+\langle M,\tilde{D}_K\rangle \langle \tilde{D}_0,N\rangle\\
    =&2\langle M,\tilde{D}_KL^\top NL\tilde{D}_K\rangle+\langle M,\tilde{D}_K\rangle\langle \tilde{D}_K,N\rangle \\
    =&\text{svec}(M)^\top (2\tilde{D}_KL^\top \otimes_s \tilde{D}_KL^\top\\
    &+\text{svec}(\tilde{D}_K)\text{svec}(\tilde{D}_K)^\top)\text{svec}(N),
\end{align*}
where we make use of the Lyapunov equation \eqref{lyap2}. Thus we get
\begin{align}\label{eq239}
    \mathbb{E}_{\vartheta\sim\tilde{\rho}_K}&[\phi(\vartheta)\phi(\vartheta')^\top]\nonumber\\
    =2&\tilde{D}_KL^\top \otimes_s \tilde{D}_KL^\top+\text{svec}(\tilde{D}_K)\text{svec}(\tilde{D}_K)^\top.
\end{align}
Therefore, combining \eqref{eq226} and \eqref{eq239}, we have
\begin{align*}
    A_K&=2(\tilde{D}_K\otimes_s \tilde{D}_K-\tilde{D}_KL^\top \otimes_s \tilde{D}_KL^\top)\\
    &=2(\tilde{D}_K\otimes_s \tilde{D}_K)(I-L^\top\otimes_s L^\top),
\end{align*}
where in the last equality we use the fact that
\begin{align*}
    (A\otimes_s B)(C\otimes_s D)=\frac{1}{2}(AC\otimes_s BD+AD\otimes_s BC)
\end{align*}
for any matrices $A,B,C,D$. Since $\rho(L)<1$, then $I-L^\top\otimes_s L^\top$ is positive definite, which further implies $A_K$ is invertible.

From Bellman equation of $Q_K$, we have
 \begin{align*}
     \langle \phi(x,u), \text{svec}(\Omega_K)\rangle &=c(x,u)-J(K)\\
     &+\langle \mathbb{E}[\phi(x',u')|x,u],\text{svec}(\Omega_K)\rangle.
 \end{align*}
Multiply each side by $\phi(x,u)$ and take a expectation with respect to $(x,u)$, we get
\begin{align*}
    \mathbb{E}[\phi(x,u)(\phi(x,u)-&\mathbb{E}[\phi(x',u')|x,u])^\top]\text{svec}(\Omega_K)\\
    =&\mathbb{E}[\phi(x,u)(c(x,u)-J(K))].
\end{align*}
We further have
\begin{align*}
    &\mathbb{E}[\phi(x,u)(\phi(x,u)-\mathbb{E}[\phi(x',u')|x,u])^\top]\\
    =&\mathbb{E}[\phi(x,u)(\phi(x,u)-\phi(x',u'))^\top]\\
    =&A_K,
\end{align*}
where the first equality comes from the low of total expectation and
\begin{align*}
    \mathbb{E}[\phi(x,u)(c(x,u)-J(K))]=b_K.
\end{align*}
Therefore, we get
 \begin{align*}
     A_K\text{svec}(\Omega_K)=b_K,
 \end{align*}
which implies $\omega^\ast_K=\text{svec}(\Omega_K)$. Thus we conclude our proof.
\end{proof}
\noindent \textbf{Proof of Proposition \ref{pa1.2.1}}:
\begin{proof}
Since $D_{K_t}$ satisfies the Lyapunov equation defined in \eqref{lyap1}, we have
\begin{align*}
    D_{K_t}=\sum\limits_{k=0}^{\infty}(A-BK_t)^kD_{\sigma}((A-BK_t)^\top)^k.
\end{align*}
From Assumption \ref{a2}, we know that $\rho(A-BK_t)\leq\rho<1$. Thus for any $\epsilon>0$, there exists a sub-multiplicative matrix norm $\Vert \cdot \Vert_{\ast}$ such that
\begin{align*}
    \Vert A-BK_t\Vert_{\ast}\leq \rho(A-BK_t)+\epsilon.
\end{align*}
Choose $\epsilon=\frac{1-\rho}{2}$, we get
\begin{align*}
    \Vert A-BK_t\Vert_{\ast}\leq \frac{1+\rho}{2}<1.
\end{align*}
Therefore, we can bound the norm of $D_{K_t}$ by
\begin{align*}
    \Vert D_{K_t}\Vert_{\ast} &\leq \sum\limits_{k=0}^\infty \Vert A-BK_t\Vert^{2k}_{\ast}\Vert D_{\sigma}\Vert_{\ast}\\
    &\leq \Vert D_{\sigma}\Vert_{\ast}\sum\limits_{k=0}^{\infty}(\frac{1+\rho}{2})^{2k}\\
    &\leq \Vert D_{\sigma} \Vert_{\ast}\frac{1}{1-(\frac{1+\rho}{2})^2}.
\end{align*}
Since all norms are equivalent on the finite dimensional Euclidean space, there exists a constant $c_1$ satisfies
\begin{align*}
    \Vert D_{K_t}\Vert \leq \frac{c_1}{1-(\frac{1+\rho}{2})^2}\Vert D_{\sigma}\Vert,
\end{align*}
which concludes our proof.
\end{proof}
\noindent \textbf{Proof of Proposition \ref{new_bounded}}:
\begin{proof}
    We first bound $\mathbb{E}[c_t^2]$. Note that from the proof of Proposition \ref{pro3}, we have $\vartheta_t=(x_t^\top,u_t^\top)^\top\sim \mathcal{N}(0,\Tilde{D}_{K_t})$, where $\Tilde{D}_{K_t}$ is upper bounded by \eqref{tilde_D_K}. Combining with Proposition \ref{pa1.2.1}, we know that $\Tilde{D}_{K_t}$ is norm bounded.
Define
\begin{align*}
    \Sigma:=\begin{bmatrix}
        Q &  \\  & R
    \end{bmatrix}.
\end{align*}
It holds that 
\begin{align*}
    c_t=x_t^\top Qx_t+u_t^\top Ru_t=\vartheta^\top \Sigma \vartheta.
\end{align*}
Then we have
\begin{align*}
    \mathbb{E}[c_t^2]=&\mathbb{E}[(\vartheta^\top \Sigma \vartheta)^2]\\
    =& \text{Var}(\vartheta^\top \Sigma \vartheta)+[\mathbb{E}(\vartheta^\top \Sigma \vartheta)]^2\\
    =& 2\text{Tr}(\Sigma \Tilde{D}_{K_t}\Sigma \Tilde{D}_{K_t})+(\text{Tr}(\Sigma \Tilde{D}_{K_t}))^2,
\end{align*}
where we use the fact that if $\vartheta\sim\mathcal{N}(\mu,D)$ is a multivariate Gaussian distribution and $\Sigma$ is a symmetric matrix, we have (\cite{rencher2008linear})
\begin{align*}
    \mathbb{E}[\vartheta^\top \Sigma \vartheta]=&\text{Tr}(\Sigma D)+\mu^\top \Sigma \mu,\\
    \text{Var}(\vartheta^\top \Sigma \vartheta)=& 2\text{Tr}(\Sigma D\Sigma D)+4\mu^\top D\Sigma D \mu.
\end{align*}
Since $\Sigma$ and $\Tilde{D}_{K_t}$ are both uniform bounded, $\mathbb{E}[c_t^2]$ is also uniform bounded.

It reminds to bound $\mathbb{E}[\Vert \phi(x_t,u_t)\Vert^2]$. We know that
    \begin{align*}
        \Vert \phi(x_t,u_t)\Vert^2=&\langle \text{svec}(\vartheta_t \vartheta_t^\top), \text{svec}(\vartheta_t \vartheta_t^\top)\rangle\\
        =&\Vert \vartheta_t\vartheta_t^\top\Vert_{\text{F}}^2\\
        =&\sum\limits_{1\leq i,j\leq d+k}(\vartheta_t^i\vartheta_t^j)^2,
    \end{align*}
where $\vartheta_t^i$ and $\vartheta_j^j$ are i-th and j-th component of $\vartheta_t$ respectively. Therefore, we can further get
    \begin{align*}
        \mathbb{E}[\Vert \phi(x_t,u_t)\Vert^2]=& \sum\limits_{1\leq i,j\leq d+k}\mathbb{E}(\vartheta_t^i\vartheta_t^j)^2.
    \end{align*}
It can be shown that
    \begin{align*}
        \vartheta_t^i\vartheta_t^j=\frac{1}{4}(\vartheta_t^i+\vartheta_t^j)^2-\frac{1}{4}(\vartheta_t^i-\vartheta_t^j)^2.
    \end{align*}
Since both $\vartheta_t^i$ and $\vartheta_t^j$ are univariate Gaussian distributions, we have
\begin{align*}
    \vartheta_t^i\vartheta_t^j=\frac{\text{Var}(\vartheta_t^i+\vartheta_t^j)}{4}X-\frac{\text{Var}(\vartheta_t^i-\vartheta_t^j)}{4}Y,
\end{align*}
where $X,Y\sim \chi_1^2$ and we use the fact that the squared of a standard Gaussian random variable has a chi-squared distribution. From $\Vert\Tilde{D}_{K_t}\Vert_{\text{F}}$ is bounded, we know that $\text{Var}(\vartheta_t^i+\vartheta_t^j)$ and $\text{Var}(\vartheta_t^i-\vartheta_t^j)$ are both bounded. Define $c_1:=\frac{\text{Var}(\vartheta_t^i+\vartheta_t^j)}{4}$ and $c_2:=\frac{\text{Var}(\vartheta_t^i-\vartheta_t^j)}{4}$, we can show have
\begin{align*}
    \mathbb{E}[(\vartheta_t^i\vartheta_t^j)^2]=&\text{Var}(\vartheta_t^i\vartheta_t^j)+(\mathbb{E}(\vartheta_t^i\vartheta_t^j))^2\\
    =&\text{Var}(c_1X-c_2Y )+(\mathbb{E}[c_1X-c_2Y])^2.
\end{align*}
Since $EX=EY=1, \text{Var}(X)=\text{Var}(Y)=2$, it holds that
\begin{align*}
    \mathbb{E}[(\vartheta_t^i\vartheta_t^j)^2]=&\text{Var}(c_1X-c_2Y )+(\mathbb{E}[c_1X-c_2Y])^2\\
    =&2c_1^2+2c_2^2-2c_1c_2\text{Cov}(X,Y)+(c_1-c_2)^2\\
    \leq& 4c_1^2+4c_2^2+2c_1c_2\sqrt{\text{Var}(X)\text{Var}(Y)}\\
    =&4c_1^2+4c_2^2+4c_1c_2.
\end{align*}
Therefore, we get
\begin{align*}
    \mathbb{E}[\Vert \phi(x_t,u_t)\Vert^2]&= \sum\limits_{1\leq i,j\leq d+k}\mathbb{E}(v_iv_j)^2\\
    &\leq (d+k)^2(4c_1^2+4c_2^2+4c_1c_2),
\end{align*}
which is bounded.

Overall, we have shown that there exists a constant $C>0$ such that
\begin{align*}
    \mathbb{E}[c^2_t]\leq C,\\
    \mathbb{E}[\Vert \phi(x_t,u_t)\Vert^2]\leq C.
\end{align*}
\end{proof}
\noindent \textbf{Proof of Proposition \ref{cost_function}}:
\begin{proof}
It can be shown that
    \begin{align*}
    J(K_t)=&\mathbb{E}_{(x_t,u_t)}[c(x_t,u_t)] \\
    =&\mathbb{E}[x_t^\top Qx_t+u_t^\top Ru_t] \\
    =&\mathbb{E}[x_t^\top Qx_t+(-Kx_t+\sigma \zeta_t)^\top R(-Kx_t+\sigma\zeta_t)] \\
    =&\mathbb{E}_{x_t\sim \rho_{K_t}}\mathbb{E}_{\zeta_t\sim \mathcal{N}(0,I_k)}[x_t^\top(Q+K_t^\top RK_t)x_t\\
    &-\sigma x_t^\top K_t^\top R\zeta_t -\sigma\zeta_t^\top RK_tx_t+\sigma^2\zeta_t^\top R\zeta_t ] \\
    =&\mathbb{E}_{x_t\sim\rho_{K_t}}[x_t^\top (Q+K_t^\top RK_t)x_t]+\sigma^2\text{Tr}(R)  \\
    =&\text{Tr}((Q+K_t^\top RK_t)D_{K_t})+\sigma^2\text{Tr}(R)\\
    \leq & \Vert (Q+K_t^\top RK_t)D_{K_t}\Vert_{\text{F}}+\sigma^2\text{Tr}(R)\\
    \leq & \Vert Q\Vert_{\text{F}}+\Vert K_t \Vert^2_{\text{F}}+\Vert R\Vert_{\text{F}}+\Vert D_{K_t}\Vert_{\text{F}}+\sigma^2\text{Tr}(R)\\
    \leq & \Vert Q\Vert_{\text{F}}+d\Bar{K}^2+\Vert R\Vert_{\text{F}}+\sqrt{d}\Vert D_{K_t}\Vert +\sigma^2\text{Tr}(R)\\
    \leq &\Vert Q\Vert_{\text{F}}+d\Bar{K}^2+\Vert R\Vert_{\text{F}}+\sigma^2\text{Tr}(R)\\
    &+\frac{c_1\sqrt{d}}{1-(\frac{1+\rho}{2})^2}\Vert D_{\sigma}\Vert\\
    := & U,
\end{align*}
where the last inequality comes from Proposition \ref{pa1.2.1}.
\end{proof}
\noindent \textbf{Proof of Proposition \ref{p1}}:
\begin{proof}
\begin{align*}
    &|J(K_{t+1})-J(K_t)|\\
    =&|\text{Tr}((P_{K_{t+1}}-P_{K_t})D_{\sigma})| \\
    \leq& d\Vert D_{\sigma} \Vert \Vert P_{K_{t+1}}-P_{K_t}\Vert\\
    \leq& 6d\Vert D_{\sigma} \Vert \sigma_{\text{min}}^{-1}(D_0)\Vert D_{K_t}\Vert \Vert K_t\Vert \Vert R\Vert(\Vert K_t\Vert \Vert B\Vert \Vert A-BK_t\Vert\\
    &+\Vert K_t\Vert \Vert B\Vert +1)\Vert K_{t+1}-K_t\Vert\\
    \leq& 6c_1d\bar{K}\sigma_{\text{min}}^{-1}(D_0)\frac{\Vert D_{\sigma} \Vert^2}{1-(\frac{1+\rho}{2})^2}\Vert R\Vert(\bar{K}\Vert B\Vert(\Vert A\Vert \\
    &+\bar{K}\Vert B\Vert+1) +1)\Vert K_{t+1}-K_t\Vert\\
    =&l_1\Vert K_{t+1}-K_t\Vert,
\end{align*}
where the second inequality is due to the perturbation of $P_K$ in Lemma \ref{l2} and
\begin{align*}
    l_1:=&6c_1d\bar{K}\sigma_{\text{min}}^{-1}(D_0)\frac{\Vert D_{\sigma} \Vert^2}{1-(\frac{1+\rho}{2})^2}\Vert R\Vert(\bar{K}\Vert B\Vert\cdot \\
    &(\Vert A\Vert+\bar{K}\Vert B\Vert+1) +1).
\end{align*}
Thus we finish our proof.
\end{proof}
\noindent \textbf{Proof of Proposition \ref{p2}}:
\begin{proof}
From Proposition \ref{pro3}, we know that 
\begin{align*}
    A_{K_t}=2(\tilde{D}_{K_t}\otimes_s\tilde{D}_{K_t})(I-L^\top\otimes_sL^\top).
\end{align*}
By Assumption \ref{a2}, we have $\rho(L)=\rho(A-BK_t)\leq \rho<1$.
Then we have 
\begin{align*}
    \Vert A_{K_t}^{-1}\Vert &=\frac{1}{2}\Vert (I-L^\top\otimes_sL^\top)^{-1}(\tilde{D}_{K_t}\otimes_s\tilde{D}_{K_t})^{-1}\Vert\\
    &\leq \frac{1}{2}\Vert (I-L^\top\otimes_sL^\top)^{-1}\Vert \Vert(\tilde{D}_{K_t}\otimes_s\tilde{D}_{K_t})^{-1}\Vert\\
    &\leq \frac{1}{2(1-\rho^2)}\Vert \tilde{D}_{K_t}^{-1}\Vert^2\\
    &= \frac{1}{2(1-\rho^2)\sigma^2_{\text{min}}(\tilde{D}_{K_t})}.\\
\end{align*}
To bound $\sigma_{\text{min}}(\tilde{D}_{K_t})$, for any $a\in\mathbb{R}^d$ and $b\in\mathbb{R}^k$, we have
\begin{align*}
    &\begin{pmatrix}
    a^\top & b^\top
    \end{pmatrix} \tilde{D}_{K_t}\begin{pmatrix}
    a\\b
    \end{pmatrix}\\
    =&\mathbb{E}_{(x,u)\sim\mathcal{N}(0,\tilde{D}_{K_t})}[\begin{pmatrix}
    a^\top & b^\top
    \end{pmatrix} \begin{pmatrix}
    x \\ u
    \end{pmatrix}\begin{pmatrix}
    x^\top & u^\top
    \end{pmatrix}\begin{pmatrix}
    a\\b
    \end{pmatrix}]\\
    =&\mathbb{E}_{(x,u)\sim\mathcal{N}(0,\tilde{D}_{K_t})}[((a^\top-b^\top K_t)x+\sigma b^\top \zeta)\cdot \\
    &((a^\top-b^\top K_t)x+\sigma b^\top \zeta)^\top]\\
    =&\mathbb{E}_{x\sim\mathcal{N}(0,D_{K_t}), \zeta\sim\mathcal{N}(0,I_k)}[(a^\top -b^\top K_t)xx^\top (a-K_t^\top b)\\
    &+\sigma^2 b^\top \zeta \zeta^\top b]\\
    \ge &\sigma_{\text{min}}(D_{K_t})\Vert a-K_t^\top b\Vert^2+\sigma^2\Vert b\Vert^2.\\
\end{align*}
For $\Vert a-K_t^\top b\Vert^2$, we have
\begin{align*}
    \Vert a-K_t^\top b\Vert^2\ge &\Vert a\Vert^2+\Vert K^\top_tb\Vert^2-2\Vert a\Vert \Vert K_t^\top\Vert\Vert b\Vert\\
    \ge&\Vert a\Vert^2-2\bar{K}\Vert a\Vert\Vert b\Vert\\
    \ge &\Vert a\Vert^2-\frac{1}{2}(\Vert a\Vert^2+4\bar{K}^2\Vert b\Vert^2)\\
    =&\frac{1}{2}\Vert a\Vert^2-2\bar{K}^2\Vert b\Vert^2.
\end{align*}
Hence we get
\begin{align*}
    &\begin{pmatrix}
    a^\top & b^\top
    \end{pmatrix} \tilde{D}_{K_t}\begin{pmatrix}
    a\\b
    \end{pmatrix}\\
    \ge &\sigma_{\text{min}}(D_{K_t})\Vert a-K_t^\top b\Vert^2+\sigma^2\Vert b\Vert^2\\
    \ge &\sigma_{\text{min}}(D_{K_t})(\frac{1}{2}\Vert a\Vert^2-2\bar{K}^2\Vert b\Vert^2)+\sigma^2\Vert b\Vert^2\\
    \ge &\min \{\sigma_{\min} (D_0),\frac{\sigma^2}{4\bar{K}^2} \}(\frac{1}{2}\Vert a\Vert^2-2\bar{K}^2\Vert b\Vert^2)+\sigma^2\Vert b\Vert^2\\
    \ge &\min \{\frac{\sigma_{\min} (D_0)}{2},\frac{\sigma^2}{8\bar{K}^2},\frac{\sigma^2}{2} \}(\Vert a\Vert^2+\Vert b\Vert^2).
\end{align*}
Thus we have 
\begin{align*}
    \sigma_{\text{min}}(\tilde{D}_{K_t})\ge \min \{\frac{\sigma_{\min} (D_0)}{2},\frac{\sigma^2}{8\bar{K}^2},\frac{\sigma^2}{2} \}>0,
\end{align*}
which further implies
\begin{align*}
    \Vert A^{-1}_{K_t}\Vert &\leq \frac{1}{2(1-\rho^2)\sigma^2_{\min}(\tilde{D}_{K_t})}\\
    &\leq \frac{1}{2(1-\rho^2)(\min \{\frac{\sigma_{\min} (D_0)}{2},\frac{\sigma^2}{8\bar{K}^2},\frac{\sigma^2}{2} \})^2}.
\end{align*}
We define
\begin{align*}
    \lambda:=2(1-\rho^2)(\min \{\frac{\sigma_{\min} (D_0)}{2},\frac{\sigma^2}{8\bar{K}^2},\frac{\sigma^2}{2} \})^2
\end{align*}
such that we get
\begin{align*}
    \sigma_{\text{min}}(A_{K_t})\ge \lambda,
\end{align*}
which concludes the proof.
\end{proof}
\noindent \textbf{Proof of Proposition \ref{p3}}:
\begin{proof}
\begin{align*}
    &\Vert\omega^\ast_t-\omega^\ast_{t+1}\Vert\\
    =& \Vert \text{svec}(\Omega_{K_t}-\Omega_{K_{t+1}})\Vert \\
    =&\Vert \Omega_{K_t}-\Omega_{K_{t+1}}\Vert_{\text{F}}\\
=&\Vert \begin{bmatrix}
A^\top (P_{K_t}-P_{K_{t+1}}) A & A^\top (P_{K_t}-P_{K_{t+1}})B \\ B^\top (P_{K_t}-P_{K_{t+1}}) A & B^\top (P_{K_t}-P_{K_{t+1}})B
\end{bmatrix}\Vert_{\text{F}}\\
=&\Vert A^\top (P_{K_t}-P_{K_{t+1}}) A\Vert_{\text{F}}+\Vert A^\top (P_{K_t}-P_{K_{t+1}}) B\Vert_{\text{F}}\\
&+\Vert B^\top (P_{K_t}-P_{K_{t+1}}) A\Vert_{\text{F}} + \Vert B^\top (P_{K_t}-P_{K_{t+1}}) B\Vert_{\text{F}}\\
\leq& d^{\frac{3}{2}}(\Vert A\Vert+\Vert B\Vert)^2\Vert P_{K_t}-P_{K_{t+1}}\Vert\\
\leq& 6d^{\frac{3}{2}}(\Vert A\Vert+\Vert B\Vert)^2\sigma_{\text{min}}^{-1}(D_0)\Vert D_{K_t}\Vert \Vert K_t\Vert \Vert R\Vert \cdot  \\
&(\Vert K_t\Vert \Vert B\Vert \Vert A-BK_t\Vert+\Vert K_t\Vert \Vert B\Vert +1)\Vert K_{t+1}-K_t\Vert\\
\leq& 6c_1d^{\frac{3}{2}}(\Vert A\Vert+\Vert B\Vert)^2 \sigma_{\text{min}}^{-1}(D_0)\frac{\Vert D_{\sigma}\Vert}{1-(\frac{1+\rho}{2})^2}\bar{K}\Vert R\Vert \cdot \\
&(\bar{K}\Vert B\Vert(\Vert A\Vert +\bar{K}\Vert B\Vert+1) +1)\Vert K_{t+1}-K_t\Vert\\
=&l_2 \Vert K_{t+1}-K_t\Vert,
\end{align*}
where
\begin{align}
    l_2:=&6c_1d^{\frac{3}{2}}\bar{K}(\Vert A\Vert+\Vert B\Vert)^2 \sigma_{\text{min}}^{-1}(D_0)\frac{\Vert D_{\sigma}\Vert\Vert R\Vert}{1-(\frac{1+\rho}{2})^2}\cdot \nonumber
    \\&(\bar{K}\Vert B\Vert(\Vert A\Vert +\bar{K}\Vert B\Vert+1) +1).
\end{align}
\end{proof}
\section{Proof of Auxiliary Lemmas}\label{appendix3}
The following lemmas are well known and have been established in several papers \cite{yang2019provably,fazel2018global}. We include the proof here only for completeness.

\noindent \textbf{Proof of Lemma \ref{pro1}}:
\begin{proof}
Since we focus on the family of linear-Gaussian policies defined in \eqref{policy}, we have
\begin{align}
    J(K)=&\mathbb{E}_{(x,u)}[c(x,u)] \nonumber\\
    =&\mathbb{E}_{(x,u)}[x^\top Qx+u^\top Ru] \nonumber\\
    =&\mathbb{E}_{(x,u)}[x^\top Qx+(-Kx+\sigma \zeta)^\top R(-Kx+\sigma\zeta)] \nonumber \\
    =&\mathbb{E}_{x\sim \rho_K}\mathbb{E}_{\zeta\sim I_k}[x^\top(Q+K^\top RK)x-\sigma x^\top K^\top R\zeta \nonumber \\
    -&\sigma\zeta^\top RKx+\sigma^2\zeta^\top R\zeta ] \nonumber \\
    =&\mathbb{E}_{x\sim\rho_K}[x^\top (Q+K^\top RK)x]+\sigma^2\text{Tr}(R) \nonumber \\
    =&\text{Tr}((Q+K^\top RK)D_K)+\sigma^2\text{Tr}(R). \label{300}
\end{align}
Furthermore, for $K\in\mathbb{R}^{k\times d}$ such that $\rho(AB-K)<1$ and positive definite matrix $S\in\mathbb{R}^{d\times d}$, we define the following two operators
\begin{align}\label{operator}
    \Gamma_K(S)=\sum\limits_{t\ge 0}(A-BK)^tS[(A-BK)^t]^\top, \nonumber \\ \Gamma^\top_K(S)=\sum\limits_{t\ge 0}[(A-BK)^t]^\top S(A-BK)^t.
\end{align}
Hence, $\Gamma_K(S)$ and $\Gamma_K^\top (S)$ satisfy Lyapunov equations
\begin{align}
    &\Gamma_K(S)=S+(A-BK)\Gamma_K(S)(A-BK)^\top, \label{op1}\\
    &\Gamma^\top_K(S)=S+(A-BK)^\top\Gamma^\top_K(S)(A-BK) \label{op2}
\end{align}
respectively. Therefore, for any positive definite matrices $S_1$ and $S_2$, we get
\begin{align*}
    \text{Tr}(S_1\Gamma_K(S_2))&=\sum\limits_{t\ge 0}\text{Tr}(S_1(A-BK)^tS_2[(A-BK)^t]^\top)\\
    &=\sum\limits_{t\ge 0}\text{Tr}([(A-BK)^t]^\top S_1(A-BK)^tS_2)\\
    &=\text{Tr}(\Gamma_K^\top(S_1)S_2).
\end{align*}
Combining \eqref{lyap1}, \eqref{lyap2}, \eqref{op1} and \eqref{op2}, we know that
\begin{align}\label{operatorpk}
    D_K=\Gamma_K(D_{\sigma}), \quad P_K=\Gamma^\top_K(Q+K^\top RK).
\end{align}
Thus \eqref{300} implies
\begin{align*}
    J(K)&=\text{Tr}((Q+K^\top RK)D_K)+\sigma^2\text{Tr}(R)\\
    &=\text{Tr}((Q+K^\top RK)\Gamma_K(D_{\sigma}))+\sigma^2\text{Tr}(R)\\
    &=\text{Tr}(\Gamma_K^\top (Q+K^\top RK)D_{\sigma})+\sigma^2\text{Tr}(R)\\
    &=\text{Tr}(P_KD_{\sigma})+\sigma^2\text{Tr}(R).
\end{align*}
It remains to establish the gradient of $J(K)$. Based on \eqref{300}, we have
\begin{align*}
    \nabla_K J(K)=&\nabla_K \text{Tr}((Q+K^\top RK)C))|_{C=D_K}\\
    &+\nabla_K\text{Tr}(CD_K)|_{C=Q+K^\top RK},
\end{align*}
where we use $C$ to denote that we compute the gradient with respect to $K$ and then substitute the expression of $C$. Hence we get
\begin{align}\label{302}
    \nabla_K J(K)=2RKD_K+\nabla_K \text{Tr}(C_0D_K)|_{C_0=Q+K^\top RK}.
\end{align}
Furthermore, we have
\begin{align*}
    &\nabla_K \text{Tr}(C_0D_K)\\
    =&\nabla_K \text{Tr}(C_0\Gamma_K(D_{\sigma}))\\
    =&\nabla_K \text{Tr}(C_0D_{\sigma}+C_0(A-BK)\Gamma_K(D_{\sigma})(A-BK)^\top)\\
    =&\nabla_K \text{Tr}(C_0D_{\sigma})\\
    &+\nabla_K \text{Tr}((A-BK)^\top C_0(A-BK)\Gamma_K(D_{\sigma}))\\
    =&-2B^\top C_0(A-BK)\Gamma_K(D_{\sigma})\\
    &+\nabla_K \text{Tr}(C_1 \Gamma_K(D_{\sigma}))|_{C_1=(A-BK)^\top C_0(A-BK)}.
\end{align*}
Then it reduces to compute $\nabla_K \text{Tr}(C_1 \Gamma_K(D_{\sigma}))|_{C_1=(A-BK)^\top C_0(A-BK)}$. Applying this iteration for $n$ times, we get
\begin{align}\label{301}
    &\nabla_K \text{Tr}(C_0D_K) \nonumber \\
    =&-2B^\top\sum\limits_{t=0}^nC_t(A-BK)\Gamma_K(D_{\sigma})\nonumber \\
    &+\nabla_K \text{Tr}(C_n \Gamma_K(D_{\sigma}))|_{C_n=[(A-BK)^n]^\top C_0(A-BK)^n}.
\end{align}
Meanwhile, by Lyapunov equation defined in \eqref{lyap3}, we have
\begin{align*}
    \sum\limits_{t=0}^{\infty}C_t=&\sum\limits_{t=0}^{\infty}[(A-BK)^t]^\top (Q+K^\top RK)(A-BK)^t\\
    =&P_K.
\end{align*}
Since $\rho(A-BK)<1$, we further get
\begin{align*}
    &\lim\limits_{n\to\infty}\text{Tr}(C_n\Gamma_K(D_{\sigma}))\\
    \leq &\lim\limits_{n\to\infty}\Vert (Q+K^\top RK)\Vert \rho(A-BK)^{2n}\text{Tr}(\Gamma_K(D_{\sigma}))\\
    =& 0.
\end{align*}
Thus by letting $n$ go to infinity in \eqref{301}, we get
\begin{align*}
    &\nabla_K \text{Tr}(C_0D_K)|_{C_0=Q+K^\top RK}\\
    =&-2B^\top P_K(A-BK)\Gamma_K(D_{\sigma})\\
    =&-2B^\top P_K(A-BK)D_K.
\end{align*}
Hence, combining \eqref{302}, we have
\begin{align*}
    \nabla_K J(K)&=2RKD_K-2B^\top P_K(A-BK)D_K\\
    &=2[(R+B^\top P_KB)K-B^\top P_KA]D_K,
\end{align*}
which concludes our proof.
\end{proof}
\noindent \textbf{Proof of Lemma \ref{pro2}}:
\begin{proof}
By definition, we have the state-value function as follows
\begin{align}
    V_{\theta}(x):&=\sum\limits_{t=0}^{\infty}\mathbb{E}_{\theta}[(c(x_t,u_t)-J(\theta))|x_0=x]\nonumber \\
    &=\mathbb{E}_{u\sim\pi_{\theta}(\cdot|x)}[Q_{\theta}(x,u)],
\end{align}
Therefore, we have
\begin{align}
    &V_K(x)\nonumber \\
    =&\sum\limits_{t=0}^{\infty}\mathbb{E}[c(x_t,u_t)-J(K)|x_0=x,u_t=-Kx_t+\sigma \zeta_t]\nonumber \\
    =&\sum\limits_{t=0}^{\infty}\mathbb{E}\{[x_t^\top (Q+K^\top RK)x_t]+\sigma^2\text{Tr}(R)-J(K)\}.\label{303}
\end{align}
Combining the linear dynamic system in \eqref{eq:6} and the form of \eqref{303}, we see that $V_K(x)$ is a quadratic function, which can be denoted by
\begin{align*}
    V_K(x)=x^\top P_Kx+C_K,
\end{align*}
where $P_K$ is defined in \eqref{lyap3} and $C_K$ only depends on $K$. Moreover, by definition, we know that $\mathbb{E}_{x\sim\rho_K}[V_K(x)]=0$, which implies
\begin{align*}
    \mathbb{E}_{x\sim\rho_K}[x^\top P_Kx+C_K]=\text{Tr}(P_KD_K)+C_K=0.
\end{align*}
Thus we have $C_K=-\text{Tr}(P_KD_K)$. Hence, the expression of $V_K(x)$ is given by
\begin{align*}
    V_K(x)=x^\top P_K x-\text{Tr}(P_KD_K).
\end{align*}
Therefore, the action-value function $Q_K(x,u)$ can be written as
\begin{align*}
    Q(x,u)=&c(x,u)-J(K)+\mathbb{E}[V_K(x')|x,u]\\
    =&c(x,u)-J(K)+(Ax+Bu)^\top P_K(Ax+Bu)\\
    &+\text{Tr}(P_KD_0)-\text{Tr}(P_KD_K)\\
    =&x^\top Qx+u^\top Ru+(Ax+Bu)^\top P_K(Ax+Bu)\\
    &-\sigma^2\text{Tr}(R+P_KBB^\top)-\text{Tr}(P_K\Sigma_K).
\end{align*}
Thus we finish the proof.
\end{proof}
\noindent \textbf{Proof of \Cref{l2}}:
\begin{proof}
    See Lemma 5.7 in \citet{yang2019provably} for a detailed proof.
\end{proof}
\noindent \textbf{Proof of \Cref{l6}}:
\begin{proof}
    By the definition of operator in \eqref{operator} and \eqref{operatorpk}, we have
    \begin{align*}
        &x^\top P_{K'}x\\
        =&x^\top \Gamma_{K'}^\top(Q+K'^\top RK')x\\
        =&\sum\limits_{t\ge 0}x^\top [(A-BK')^t]^\top (Q+K'^\top RK')(A-BK')^tx.
    \end{align*}
    Hereafter, we define $(A-BK')^tx=x_t'$ and $u_t'=-K'x'_t$. Hence, we further have
    \begin{align*}
        x^\top P_{K'}x=&\sum\limits_{t\ge 0}x'^\top_t(Q+K'^\top RK')x'_t\\
        =&\sum\limits_{t\ge 0}(x'^\top_tQx'_t+u'^\top_tRu'_t).
    \end{align*}
    Therefore, we get
    \begin{align*}
        &x^\top P_{K'}x-x^\top P_Kx\\
        =&\sum\limits_{t\ge 0}[(x'^\top_tQx'_t+u'^\top Ru'_t)+x'^\top_tP_Kx'_t-x'^\top_tP_Kx'_t]\\
        &-x'^\top_0P_Kx'_0\\
        =&\sum\limits_{t\ge 0}[(x'^\top_tQx'_t+u'^\top Ru'_t)+x'^\top_{t+1}P_Kx'_{t+1}-x'^\top_tP_Kx'_t]\\
        =&\sum\limits_{t\ge 0}[(x'^\top_t Qx'_t+u'^\top_t Ru'_t)\\
        &+[(A-BK')x'_t]^\top P_K(A-BK')x'_t-x'_tP_Kx'_t]\\
        =&\sum\limits_{t\ge 0}\{x'^\top_t[Q+(K'-K+K)^\top R(K'-K+K)]x'_t\\
        &+x'^\top_t [A-BK-B(K'-K)^\top P_K [A-BK\\
        &-B(K'-K)]x'_t-x'_tP_Kx'_t\}\\
        =&\sum\limits_{t\ge 0}\{2x_t^\top(K'-K)^\top [(R+B^\top P_KB)K-B^\top P_KA]x'_t\\
        &+x'^\top_t(K'-K)^\top (R+B^\top P_KB)(K'-K)x'_t\}\\
        =&\sum\limits_{t\ge 0}[2x'^\top_t(K'-K)^\top E_Kx'_t+x'^\top_t(K'-K)^\top \cdot \\ &(R+B^\top P_KB)(K'-K)x'_t].
    \end{align*}
    Define
    \begin{align}\label{ak2}
        A_{K, K'}&(x):=2x^\top(K'-K)^\top E_Kx  \nonumber \\
        +&x^\top(K'-K)^\top(R+B^\top P_KB)(K'-K)x.
    \end{align}
    Then, from the expression of $J(K)$ in \eqref{eq.cost_formula}, we have
    \begin{align*}
        &J(K')-J(K)\\
        =&\mathbb{E}_{x\sim\mathcal{N}(0,D_{\sigma})}[x^\top (P_{K'}-P_K)x] \\
        =&\mathbb{E}_{x'_0\sim\mathcal{N}(0,D_{\sigma})}\sum\limits_{t\ge 0}A_{K,K'}(x_t) \\
        =&\mathbb{E}_{x'_0\sim\mathcal{N}(0,D_{\sigma})}\sum\limits_{t\ge 0}[2x'^\top_t(K'-K)^\top E_Kx'_t\\
        &+x'^\top_t(K'-K)^\top (R+B^\top P_KB)(K'-K)x'_t] \nonumber \\
        =&\text{Tr}(2\mathbb{E}_{x'_0\sim\mathcal{N}(0,D_{\sigma})}[\sum\limits_{t\ge 0}x'^\top_tx'_t](K'-K)^\top E_K)\nonumber \\
        &+\text{Tr}(\mathbb{E}_{x'_0\sim\mathcal{N}(0,D_{\sigma})}[\sum\limits_{t\ge 0}x'^\top_tx'_t](K'-K)^\top \cdot \\
        &(R+B^\top P_K B)(K'-K))\nonumber \\
        =&-2\text{Tr}(D_{K'}(K-K')^\top E_K)\\
        &+\text{Tr}(D_{K'}(K-K')^\top (R+B^\top P_K B)(K-K')). \nonumber 
    \end{align*}
    where the last equation is due to the fact that
    \begin{align*}
        &\mathbb{E}_{x'_0\sim\mathcal{N}(0,D_{\sigma})}[\sum\limits_{t\ge 0}x'_t(x'_t)^\top]\\
        =&\mathbb{E}_{x\sim\mathcal{N}(0,D_{\sigma})}\{ \sum\limits_{t\ge 0}(A-BK')^txx^\top [(A-BK')^t]^\top \}\\
        =&\Gamma_{K'}(D_{\sigma})=D_{K'}.
    \end{align*}
    Hence, we finish our proof.
\end{proof}
\noindent \textbf{Proof of \Cref{lem:l7}}:
\begin{proof}
    By definition of $A_{K,K'}$ in \eqref{ak2}, we have
    \begin{align*}
        &A_{K, K'}(x)\\
        =&2x^\top(K'-K)^\top E_Kx\\
        &+x^\top(K'-K)^\top (R+B^\top P_KB)(K'-K)x\\
        =&\text{Tr}(xx^\top [K'-K+(R+B^\top P_K B)^{-1}E_K]^\top \cdot \\
        &(R+B^\top P_KB)[K'-K +(R+B^\top P_K B)^{-1}E_K])\\
        &-\text{Tr}(xx^\top E_K^\top (R+B^\top P_K B)^{-1}E_K)\\
        \ge & -\text{Tr}(xx^\top E_K^\top (R+B^\top P_K B)^{-1}E_K),
    \end{align*}
    where the equality is satisfied when $K'=K-(R+B^\top P_K B)^{-1}E_K$. Therefore, we have
    \begin{align*}
        J(K)-J(K^\ast)&=-\mathbb{E}_{x'_0\sim\mathcal{N}(0,D_{\sigma})}\sum\limits_{t\ge 0}A_{K,K^\ast}(x_t)\\
        &\leq \text{Tr}(D_{K^\ast}E_K^\top (R+B^\top P_K B)^{-1}E_K)\\
        &\leq \Vert D_{K^\ast}\Vert \text{Tr}(E_K^\top (R+B^\top P_K B)^{-1}E_K)\\
        &\leq \Vert D_{K^\ast}\Vert \Vert (R+B^\top P_K B)^{-1}\Vert \text{Tr}(E_K^\top E_K)\\
        &\leq \frac{1}{\sigma_{\text{min}}(R)}\Vert D_{K^\ast}\Vert\text{Tr}(E_K^\top E_K).
    \end{align*}
    Thus we complete the proof of upper bound.
    
    It remains to establish the lower bound. Since the equality is attained at $K'=K-(R+B^\top P_K B)^{-1}E_K$, we choose this $K'$ such that
    \begin{align*}
        J(K)-J(K^\ast)&\ge J(K)-J(K')\\
        &=-\mathbb{E}_{x'_0\sim\mathcal{N}(0,D_{\sigma})}[\sum\limits_{t\ge 0}A_{K,K'}(x'_t)]\\
        &=\text{Tr}(D_{K'}E^\top_K(R+B^\top P_K B)^{-1}E_K)\\
        &\ge \sigma_{\text{min}}(D_0)\Vert R+B^\top P_K B\Vert ^{-1}\text{Tr}(E_K^\top E_K).
    \end{align*}
    Overall, we have
    \begin{align*}
        &\sigma_{\text{min}}(D_0)\Vert R+B^\top P_K B\Vert ^{-1}\text{Tr}(E_K^\top E_K)\leq J(K)-J(K^\ast)\\
        &\leq \frac{1}{\sigma_{\text{min}}(R)}\Vert D_{K^\ast}\Vert\text{Tr}(E_K^\top E_K),
    \end{align*}
    which concludes our proof.
\end{proof}
\section{Experimental details}

We compare our considered single-sample two-timescale AC with two other baseline algorithms that have been analyzed in the state-of-the-art theoretical works: the double loop AC~\cite{yang2019provably} (listed in~\Cref{alg2} on the next page) and the zeroth-order method~\cite{fazel2018global} (listed in~\Cref{alg3} on the next page). 

For the considered single-sample two-timescale AC, we set for both examples $\alpha_t=\frac{0.005}{(1+t)^{0.6}}, \beta_t=\frac{0.01}{(1+t)^{0.4}}, \gamma_t=\frac{0.1}{(1+t)^{0.4}}, \sigma=1, T=10^6$. Note that multiplying small constants to these step sizes does not affect our theoretical results.

For the zeroth-order method proposed in~\citet{fazel2018global}, we set $z=5000, l=20, r=0.1$, stepsize $\eta=0.01$ and iteration number $J=1000$ for the first numerical example; while in the second example, we set $z=20000, l=50, r=0.1, \eta=0.01, J=1000$. We choose different parameters based on the trade-off between better performance and fewer sample complexity.

For the double loop AC proposed in \citet{yang2019provably}, we set for both examples $\alpha_t=\frac{0.01}{\sqrt{1+t}}, \sigma=0.2, \eta=0.05$, inner-loop iteration number $T=500000$ and outer-loop iteration number $J=100$. We note that the algorithm is fragile and sensitive to the practical choice of these parameters. Moreover, we found that it is difficult for the algorithm to converge without an accurate critic estimation in the inner-loop. In our implementation, we have to set the inner-loop iteration number to $T=500000$ to barely get the algorithm converge to the global optimum. This nevertheless demands a significant amount of computation. Higher $T$ iterations can yield more accurate critic estimation, and consequently more stable convergence, but at a price of even longer running time. We run the outer-loop for 100 times for each run of the algorithm. We run the whole algorithm 10 times independently to get the results shown in Figure~\ref{fig2}. With parallel computing implementation, it takes more than 2 weeks on our desktop workstation (Intel Xeon(R) W-2225 CPU @ 4.10GHz $\times$ 8) to finish the computation. In comparison, it takes about 0.5 hour to run the single-sample two-timescale AC and 5 hours for the zeroth-order method.
\newpage

\begin{algorithm}[H]
\caption{Double-loop Natural Actor-Critic}
\label{alg2}
\begin{algorithmic}
   \STATE Input: Initial policy $\pi_{K_0}$ such that $\rho(A-BK_0)<1$, step size $\gamma$ for policy update. 
   \WHILE{updating current policy}
   \STATE \textbf{Gradient Estimation:}
   \STATE Initialize the primal and dual variables by $v_0\in\mathcal{X}_{\Theta}$ and $\omega_0\in\mathcal{X}_{\Omega}$, respectively.
   \STATE Sample the initial state $x_0\in\mathbb{R}^d$ from stationary distribution $\rho_{K_j}$. Take action $u_0\sim \pi_{K_j}(\cdot|x_0)$ and obtain the reward $c_0$ and the next state $x_1$.
   \FOR{$i=1,2,\cdots, T$}
   \STATE Take action $u_t$ according to policy $\pi_{K_j}$, observe the reward $c_t$ and the next state $x_{t+1}$.
   \STATE $\delta_t=v^1_{t-1}-c_{t-1}+[\phi(x_{t-1},u_{t-1})-\phi(x_t,u_t)]^\top v^2_{t-1}$.
   \STATE $v^1_{t}=v^1_{t-1}-\alpha_t[\omega^1_{t-1}+\phi(x_{t-1},u_{t-1})^\top \omega^2_{t-1}]$.
   \STATE $v^2_{t}=v^2_{t-1}-\alpha_t[\phi(x_{t-1},u_{t-1})-\phi(x_t,u_t)]\cdot \phi(x_{t-1},u_{t-1})^\top\omega^2_{t-1}$.
   \STATE $\omega^1_t=(1-\alpha_t)\omega_t^1+\alpha_t(v^1_{t-1}-c_{t-1})$.
   \STATE $\omega^2_t=(1-\alpha_t)\omega^2_t+\alpha_t\delta_t\phi(x_{t-1},u_{t-1})$.
   \STATE Project $v_t$ and $\omega_t$ to $v_0\in\mathcal{X}_{\Theta}$ and $\omega_0\in\mathcal{X}_{\Omega}$.
   \ENDFOR
\STATE Return estimates:
   \begin{equation}
   \nonumber
    \widehat{v}^2 = (\sum\limits_{t=1}^T\alpha_tv^2_t)/(\sum\limits_{t=1}^T\alpha_t),\; \widehat{\Theta}=\text{smat}(\widehat{v}^2).
   \end{equation}
    \STATE \textbf{Policy Update:}
    \STATE $K_{j+1}=K_j-\eta(\widehat{\Theta}^{22}K_j-\widehat{\Theta}^{21}).$
    \STATE $j=j+1$.
\ENDWHILE
\end{algorithmic}
\end{algorithm}

\begin{algorithm}[H]
\caption{Zeroth-order Natural Policy Gradient}
\label{alg3}
\begin{algorithmic}
   \STATE Input: stabilizing policy gain $K_0$ such that $\rho(A-BK_0)<1$, number of trajectories $z$, roll-out length $l$, perturbation amplitude $r$, stepsize $\eta$
   \WHILE{updating current policy}
   \STATE \textbf{Gradient Estimation:}
   \FOR{$i=1,\cdots, z$}
   \STATE Sample $x_0$ from $\mathcal{D}$
   \STATE Simulate $K_{j}$ for $l$ steps starting from $x_0$ and observe $y_0, \cdots, y_{l-1}$ and $c_0, \cdots, c_{l-1}$. 
   \STATE Draw $U_i$ uniformly over matrices such that $\|U_i\|_F=1$, and generate a policy $K_{j,U_i}=K_j+rU_i$.
   \STATE Simulate $K_{j,U_i}$ for $l$ steps starting from $x_0$ and observe $c_0', \cdots, c_{l-1}'$.
   \STATE Calculate empirical estimates:
   \begin{equation}
   \nonumber
   \widehat{J_{K_{j}}^i}=\sum_{t=0}^{l-1} c_t,\; \widehat{\mathcal{L}_{K_{j}}^i}=\sum_{t=0}^{l-1} y_ty_t^\top,\; \widehat{J_{K_{j,U_i}}}=\sum_{t=0}^{l-1} c_t'.
   \end{equation}
   \ENDFOR
\STATE Return estimates:
   \begin{equation}
   \nonumber
    \widehat{\nabla J(K_j)} = \frac{1}{z}\sum_{i=1}^z\frac{\widehat{J_{K_{j,U_i}}}-\widehat{J_{K_{j}}^i}}{r}U_i,\; \widehat{\mathcal{L}_{K_j}}=\frac{1}{z}\sum_{i=1}^z\widehat{\mathcal{L}_{K_{j}}^i}.
   \end{equation}
    \STATE \textbf{Policy Update:}
    \STATE $K_{j+1}=K_j-\eta \widehat{\nabla J(K_j)}\widehat{\mathcal{L}_{K_j}}^{-1}.$
    \STATE $j=j+1$.
\ENDWHILE
\end{algorithmic}
\end{algorithm}

\end{document}